\documentclass{article}

\usepackage{microtype}
\usepackage{graphicx}
\usepackage{subcaption}
\usepackage{booktabs}
\usepackage{pifont}
\usepackage{hyperref}
\usepackage{xcolor}
\usepackage{siunitx}
\usepackage{tikz}
\usepackage{amsmath,amsfonts,bm,mathtools}
\usepackage{fix-cm}
\usepackage{placeins}
\usepackage{amssymb}
\usepackage{amsthm}
\usepackage[capitalize,noabbrev]{cleveref}

\usepackage[preprint]{icml2026}

\usepackage{acro}
\newcommand{\newac}[2]{\DeclareAcronym{#1}{short=#1,long=#2}}

\newac{APF}{Artificial Potential Field}
\newac{CBF}{Control Barrier Function}
\newac{CLF}{Control Lyapunov Function}
\newac{CC}{Constant Curvature}
\newac{CS}{Constant Strain}
\newac{COM}{Center of Mass}
\newac{DCSAT}{Differentiable Conservative SAT}
\newac{DCM}{Discretized Cosserat Rod Model}
\newac{DOF}{Degrees of Freedom}
\newac{EOM}{Equations of Motion}
\newac{GVS}{Geometric Variable Strain}
\newac{HOCBF}{High-Order CBF}
\newac{HOCLF}{High-Order CLF}
\newac{JIT}{Just-in-Time}
\newac{LNN}{Lagrangian Neural Network}
\newac{LSE}{Log-Sum-Exp}
\newac{LSTM}{Long Short-Term Memory}
\newac{MAE}{Mean Absolute Error}
\newac{ML}{Machine Learning}
\newac{MSE}{Mean Squared Error}
\newac{dMPC}{Differentiable Model Predictive Control}
\newac{PAC}{Piecewise Affine Curvature}
\newac{PCS}{Piecewise Constant Strain}
\newac{PCC}{Piecewise Constant Curvature}
\newac{PDE}{Partial Differential Equation}
\newac{QP}{Quadratic Program}
\newac{RL}{Reinforcement Learning}
\newac{RMSE}{Root Mean-Squared Error}
\newac{SAT}{Separating Axis Theorem}
\newac{SSAT}{Smooth SAT}
\newac{dQP}{Differentiable Quadratic Program}
\newac{poset}{Partially Ordered Set}

\theoremstyle{plain}
\newtheorem{theorem}{Theorem}[section]

\newtheorem{lemma}[theorem]{Lemma}
\newtheorem{corollary}[theorem]{Corollary}

\theoremstyle{definition}
\newtheorem{definition}[theorem]{Definition}

\theoremstyle{remark}

\usepackage[disable]{todonotes}

\DeclareRobustCommand{\PoSafeNet}{PoSafeNet}

\begin{document}

\twocolumn[
\icmltitle{PoSafeNet: Safe Learning with Poset-Structured Neural Nets}

\begin{icmlauthorlist}
  \icmlauthor{Kiwan Wong}{mit}
  \icmlauthor{Wei Xiao}{wpi,mit}
  \icmlauthor{Daniela Rus}{mit}
\end{icmlauthorlist}

\icmlaffiliation{mit}{MIT CSAIL, Cambridge, MA, USA}
\icmlaffiliation{wpi}{Worcester Polytechnic Institute, Worcester, MA, USA}

\icmlcorrespondingauthor{Kiwan Wong, Wei Xiao}{kiwan@mit.edu, wxiao3@wpi.edu}

\icmlkeywords{Safe Learning, Control Barrier Functions, Structured Safety}

\vskip 0.3in
]

\printAffiliationsAndNotice{}

\begin{abstract}
Safe learning is essential for deploying learning-based controllers in safety-critical robotic systems, yet existing approaches often enforce multiple safety constraints uniformly or via fixed priority orders, leading to infeasibility and brittle behavior.
In practice, safety requirements are heterogeneous and admit only partial priority relations, where some constraints are comparable while others are inherently incomparable.
We formalize this setting as \emph{poset-structured safety}, modeling safety constraints as a partially ordered set and treating safety composition as a structural property of the policy class.
Building on this formulation, we propose \PoSafeNet{}, a differentiable neural safety layer that enforces safety via sequential closed-form projection under poset-consistent constraint orderings, enabling adaptive selection or mixing of valid safety executions while preserving priority semantics by construction.
Experiments on multi-obstacle navigation, constrained robot manipulation, and vision-based autonomous driving demonstrate improved feasibility, robustness, and scalability over unstructured and differentiable quadratic program-based safety layers.
\end{abstract}

\section{Introduction}
\begin{figure*}[t]
  \centering
  \includegraphics[width=0.8\textwidth]{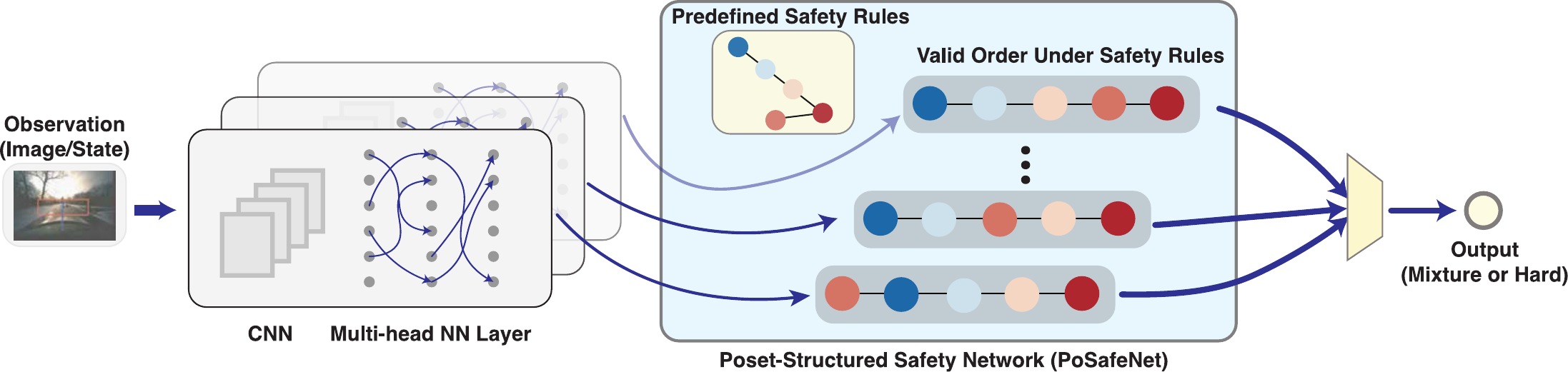}
  \caption{\textbf{Overview of \PoSafeNet{}}. A multi-head neural controller enforces safety through
  sequential projections under multiple poset-consistent constraint orderings, and adaptively
  combines the resulting safe executions via a soft mixture or a hard selection.}
  \label{fig:teaser}
\end{figure*}

Modern robot learning systems increasingly rely on large-scale training and highly expressive
function approximators to solve complex tasks such as locomotion \cite{bommasani2021opportunities}, manipulation \cite{singh2022progprompt}, and autonomous driving \cite{10611590}. Despite strong empirical performance, these models typically lack formal safety guarantees,
which limits their deployment in safety-critical domains. To address this gap, recent work has integrated formal safety certificates - most notably \ac{CBF} \cite{xiao2023barriernet, xiaoabnet2025} - into learning-based
controllers to guarantee safety through forward invariance of a prescribed safe set.

However, most existing safe learning approaches implicitly assume that multiple safety constraints
can be enforced \emph{symmetrically and simultaneously}, typically through a single optimization-based
safety layer such as a \ac{dQP} \cite{amos2017optnet} or \ac{dMPC} \cite{amos2018differentiable}. This assumption imposes a restrictive and often
unnatural \emph{representation} of safety: it forces the policy to resolve priority conflicts
implicitly through optimizations \cite{choi2025resolving, escande2014hierarchical}  that may lead to infeasibility, rather than explicitly through structure. In realistic robotic
systems, safety constraints are heterogeneous \cite{7487137, sentis2005synthesis}, arising from distinct failure modes, physical
limitations, and sensing modalities.
Treating such constraints uniformly may lead to inefficiency, infeasibility, excessive conservativeness,
or unstable learning behavior under noise and model mismatch.

A key observation motivating this work is that \emph{safety is inherently compositional}. Different
safety requirements encode distinct notions of ``what must not happen'' and are rarely mutually
interchangeable. For example, in autonomous driving, reacting to the left and right lane boundaries
corresponds to distinct safety primitives that may be individually valid yet structurally
incomparable. Crucially, neither enforcing all constraints symmetrically nor imposing an arbitrary
total ordering provides an appropriate \emph{inductive bias} for composing such heterogeneous
safety requirements.

In many practical settings, safety constraints admit only a partial order, where some priorities
are well defined while others are inherently incomparable. This raises a fundamental learning
question:
\begin{quote}
\textbf{How can we compose multiple safety-critical policies in a scalable and adaptive manner
when the underlying safety constraints admit only a partial order, rather than a single global
priority?}
\end{quote}

We argue that a central limitation of existing approaches lies not in the availability of safety
certificates themselves, but in the absence of a structural inductive bias governing how such
constraints should be composed. A partial order provides precisely the minimal structure needed:
it encodes only the necessary priority relations while leaving incomparable constraints unranked. Accordingly, we model safety composition using a \ac{poset}
over safety primitives, treated as a \emph{prior} specified by the designer rather than a learned
object.

Building on this perspective, we introduce \PoSafeNet{} (Poset-Structured Safety Network),  a learning framework that casts safety composition as a \emph{structured latent decision problem} without solving dQPs.
Given a poset over safety constraints, \PoSafeNet{} represents the space of admissible safety
executions via its poset-respecting  linear extensions. Each execution corresponds to a sequential and analytical safety projection that is provably valid by construction. A multi-head neural architecture parameterizes this discrete set of valid executions, enabling the policy to reason over alternative, priority-consistent safety behaviors.

Crucially, \PoSafeNet{} learns to \emph{select or mix} among these valid executions based on observations, yielding adaptive and context-dependent safety behavior. In effect, the policy
operates over a discrete latent space of safety-respecting executions, rather than over unconstrained control actions. This decouples the specification of safety structure from the learning of task behavior, transforming safety enforcement from a constraint satisfaction problem into a structured learning problem with strong guarantees. The \PoSafeNet{} is also more efficient than existing safety neural networks as it admits closed-form projection without solving dQPs.

We make the following \textbf{contributions}:
\begin{itemize}
\item We introduce \emph{poset-structured safety} as a structural inductive bias for learning-based
control, enabling selective and context-dependent resolution of heterogeneous safety priorities.
\item We propose \PoSafeNet{}, a neural control framework that enforces \emph{poset-respecting safety}
by construction, ensuring that lower-priority constraints are violated only when necessary to
enforce higher-priority ones.
\item We develop an efficient enforcement of \PoSafeNet{} via closed-form safety projections without solving dQPs.
\item We provide formal guarantees that the learned policy satisfies poset-respecting safety under
standard assumptions on system dynamics.
\item We demonstrate the effectiveness of \PoSafeNet{} on diverse robotic tasks, including 2D
obstacle avoidance, safe robot manipulation, and vision-based autonomous driving.
\end{itemize}

\section{Related Work}

\paragraph{Scenario-Based and High-Level Safety Frameworks.}
Scenario-based approaches are widely used for verification and validation of automated driving systems,
particularly for identifying safety-critical situations within a given operational design domain (ODD).
Zhang et al.~\cite{9763411} provided a comprehensive survey of \emph{critical scenario identification} (CSI)
methods for ADS and ADAS, focusing on scenario specification, test generation, and coverage analysis.
While these methods are essential for offline evaluation and certification,
they do not provide mechanisms for enforcing safety constraints at each control decision during
online policy execution.

Beyond scenario analysis, recent work in AI safety emphasizes explicit safety specifications
and formal guarantees beyond empirical performance.
Dalrymple et al.~\cite{dalrymple2024guaranteedsafeaiframework} propose the
\emph{Guaranteed Safe AI} framework, highlighting the roles of specifications, world models,
and verification.
Related perspectives argue that safety guarantees should be defined in a structured and
compositional manner to support formal reasoning and responsibility attribution~\cite{censi2019liability}.
However, these works focus on system-level principles and do not address how multiple safety
specifications should be composed and resolved \emph{online} at the level of continuous control
actions in learning-based controllers.

In contrast, our work focuses on \emph{online safety enforcement}.
We study how multiple safety constraints, potentially related by partial priority relations,
can be composed within a differentiable policy and resolved adaptively during execution,
enabling end-to-end learning with priority-aware safety guarantees.

\paragraph{Control Barrier Functions and Set Invariance.}
Set invariance has long been a central concept in control theory for ensuring the safety of dynamical
systems~\cite{blanchini1999set, rakovic2005invariant}.
\ac{CBF} provide a principled mechanism for enforcing forward invariance of a prescribed safe set
by imposing state-dependent constraints on the control input~\cite{ames2016control, xiao2021high}.
CBFs are closely related to classical barrier functions in optimization
\cite{prajna2007framework, boyd2004convex}
and are commonly enforced through quadratic programs that can be solved efficiently in real time.

Despite strong theoretical guarantees, \ac{CBF}-based controllers face feasibility, scalability and robustness
challenges when multiple safety constraints must be enforced.
Most existing formulations implicitly assume that constraints can be enforced either
\emph{symmetrically} as a flat intersection or according to a fixed \emph{total ordering}~\cite{10342094}.
Such assumptions can lead to infeasibility, excessive conservativeness, or brittle behavior
when safety requirements are heterogeneous or inherently incomparable.

\paragraph{Safety in Neural Network Controllers.}
In learning-based control and reinforcement learning, safety is often incorporated through
constraint penalties or barrier-based regularization, which typically provide only soft or
empirical guarantees~\cite{achiam2017constrained, tessler2018reward}.
More recently, differentiable optimization layers have enabled stronger safety guarantees
by embedding quadratic programs directly into neural networks
\cite{amos2017optnet, pereira2021safe, xiao2023barriernet, xiaoabnet2025}.
These approaches integrate CBF-based constraints with end-to-end learning,
but commonly rely on the \emph{simultaneous enforcement} of multiple constraints
or on restricting the number of constraints to maintain feasibility.

As a result, conflicting safety priorities must be resolved implicitly through optimization,
without an explicit structural inductive bias governing how such conflicts should be composed.
In contrast, our approach models safety priorities explicitly as a \emph{partial order}
and enforces constraints through \ac{poset}-consistent sequential execution.
By capturing only the minimal necessary priority relations,
we avoid artificial symmetry and brittle total-order assumptions,
and enable scalable, priority-aware online safety enforcement
within differentiable neural controllers.

\section{Problem Formulation}
\label{section: problem_formulation}

We consider a safe robot learning problem with \emph{\ac{poset}-structured safety constraints}.

\paragraph{System.}
Consider a robotic system with state $\bm{x} \in \mathbb{R}^n$ and control input
$\bm{u} \in \mathbb{R}^m$, governed by the control-affine dynamics
\begin{equation}
    \dot{\bm{x}} = f(\bm{x}) + g(\bm{x})\bm{u},
\end{equation}
where $f$ and $g$ are locally Lipschitz continuous.

\paragraph{Nominal controller.}
We are given a state-feedback nominal controller $\pi^*(\bm{x}) = \bm{u}^*$,
which provides supervision for learning but does not necessarily enforce safety.

\paragraph{Safety constraints with partial order.}
Let $\mathcal{S} = \{1,\dots,N\}$ index a set of safety constraints, each represented
by a continuously differentiable function $b_j : \mathbb{R}^n \rightarrow \mathbb{R}$,
with constraint $j$ satisfied when $b_j(\bm{x}) \ge 0$.

The constraints are endowed with a partial order $\preceq$ over $\mathcal{S}$,
where $i \prec j$ indicates that constraint $j$ has strictly higher priority than $i$.
Unlike hierarchical or lexicographic formulations that impose a total order,
a partial order allows more complex task specifications with certain constraints to remain incomparable, admitting
non-uniform enforcement when necessary to satisfy higher-priority constraints.  We define the tuple $(\mathcal{S}, \preceq)$ as a safety poset.

\paragraph{Policy class.}
We consider a neural network policy
\begin{equation}
\pi(\bm{x},\bm{z} \mid \bm{\theta}) = \bm{u},
\end{equation}
parameterized by $\bm{\theta}$, where $\bm{z}$ denotes additional observations or
contextual inputs. $\bm{x}$ may be inferenced from $\bm{z}$ using a neural network. Safety is enforced implicitly through the structure of the policy class.

\paragraph{Learning objective.}
The learning objective is to minimize the expected imitation loss
\begin{equation}
\bm{\theta}^* = \arg\min_{\bm{\theta}}
\mathbb{E}_{\bm{x},\bm{z}}\!\left[
\ell\big(\pi^*(\bm{x}), \pi(\bm{x},\bm{z} \mid \bm{\theta})\big)
\right],
\end{equation}
while satisfying the system dynamics and \ac{poset}-structured safety constraints in $\mathcal{S}$,
with safety enforced by the policy structure rather than explicit constraints
in the optimization. 

\begin{definition}[Poset-Respecting Safety]
\label{def: poset-respecting safety}
A control policy $\pi$ satisfies \emph{\ac{poset}-respecting safety} if there exists
a sequential safety enforcement ordering
$\sigma = (j_1,\dots,j_N)$ that is a linear extension of the partial order $\preceq$,
such that safety constraints are enforced according to this ordering.

Specifically, a constraint $i \in \mathcal{S}$ may be violated only when enforcing
a constraint $j \in \mathcal{S}$ with strictly higher priority, i.e., $i \prec j$.
Here, ``before'' refers to the order of constraint enforcement induced by the policy,
rather than physical time.
\end{definition}

\section{Methodology}

In this section, we present \PoSafeNet{}, a learning-based control framework
that enforces poset-respecting safety through structured policy
parameterization and differentiable sequential safety projection.
Rather than enforcing multiple safety constraints simultaneously - which
often leads to infeasible or ill-conditioned optimization problems -
\PoSafeNet{} imposes structure on how safety constraints are composed.
By treating safety as a structural constraint on the policy class,
\PoSafeNet{} enables learnable, priority-aware enforcement of multiple
control barrier functions while \emph{preserving feasibility of the safety
projection step by construction}.

\subsection{Safety Constraints as Halfspace Projections}

We first describe how individual safety constraints can be operationalized
in the control input space.
Consider a safety constraint $b_j(\bm{x}) \ge 0$ defined in
\cref{section: problem_formulation}.
Although safety constraints may be nonlinear in the state space, their
enforcement for control-affine systems induces an affine condition on the
control input.
For systems of the form
$\dot{\bm{x}} = f(\bm{x}) + g(\bm{x}) \bm{u}$,
enforcing such a constraint yields
\begin{equation}
\label{eqn:constraint_halfspace}
    A_j(\bm{x}) \bm{u} \ge c_j(\bm{x}),
\end{equation}
where $A_j(\bm{x}) \in \mathbb{R}^{1 \times m}$ and
$c_j(\bm{x}) \in \mathbb{R}$ depend on the system dynamics and the
constraint function $b_j$.
This inequality defines a halfspace in the control input space.
\begin{equation}
\label{eqn:Hi_def}
H_j(\bm{x})
:=
\left\{
\bm{u}\in\mathbb{R}^m
\;\middle|\;
A_j(\bm{x}) \bm{u} \ge c_j(\bm{x})
\right\}.
\end{equation}

Safety is enforced by projecting a nominal control input
$\bm{u}_{\mathrm{nom}}$ onto the feasible halfspace associated with
constraint $j$.
Specifically, we define the projection operator
\begin{equation} \label{eqn:qp}
    \Pi_j(\bm{u})
    =
    \arg\min_{\bm{v} \in \mathbb{R}^m}
    \|\bm{v} - \bm{u}\|^2
    \quad
    \text{s.t. }
    A_j(\bm{x}) \bm{v} \ge c_j(\bm{x}).
\end{equation}

We obtain the closed-form solution of the projection operator (\ref{eqn:qp}):
\begin{equation} \label{eqn:cf}
    \Pi_j(\bm{u})
    =
    \bm{u}
    +
    \frac{\mathrm{ReLU}\!\left(
        c_j(\bm{x}) - A_j(\bm{x}) \bm{u}
    \right)}
    {\|A_j(\bm{x})\|^2}
    A_j(\bm{x})^\top .
\end{equation}
In the above, if the constraint is already satisfied, i.e., $ A_j(\bm{x}) \bm{u} \ge c_j(\bm{x})$ we have that $\Pi_j(\bm{u}) = \bm{u}$.
The ReLU term enforces the minimal correction required to restore
feasibility.
This formulation makes explicit that each safety constraint acts as an
operator on the control input rather than as a penalty term in the learning
objective.
Details on constructing $(A_j, c_j)$ from control barrier functions, as well as
the conditions under which the closed-form solution applies, are provided in
\cref{app:CBF}.

\paragraph{Learnable constraint tightness.}
The affine inequality $A_j(\bm{x}) \bm{u} \ge c_j(\bm{x})$ depends on the
choice of the class-$\mathcal{K}$ function used in the control barrier
condition.
For a standard CBF constraint of the form
\begin{equation}
    \dot b_j(\bm{x}, \bm{u}) \ge -\kappa_j\, b_j(\bm{x}),
\end{equation}
the resulting coefficients $(A_j, c_j)$ depend on the scalar gain
$\kappa_j$.
In \PoSafeNet{}, $\kappa_j$ is treated as a \emph{learnable parameter}
optimized jointly with the policy via imitation learning.
Although $\kappa_j$ is fixed across states during execution, learning its value allows the model to learn less conservative policies
and adjust the overall tightness of each safety constraint while preserving affine halfspace structure for any
$\kappa_j \ge 0$.

\subsection{Sequential Safety Projection}

While a single safety constraint can be enforced through projection,
multiple constraints are composed by applying their projection operators
sequentially.
Let $\sigma = (j_1, j_2, \dots, j_N)$ 
denote an ordering of constraint indices
in $\mathcal{S}$.
Starting from a nominal control input $\bm{u}_{\mathrm{nom}}$, we define
\begin{equation} \label{eqn:sq}
    \bm{u}^{(0)} = \bm{u}_{\mathrm{nom}}, \qquad
    \bm{u}^{(k)} = \Pi_{j_k}\!\left(\bm{u}^{(k-1)}\right),
    \quad k = 1, \dots, N.
\end{equation}
where $\Pi_{j_k}$ is the closed-form projection operator defined in
(\ref{eqn:cf}).
The final control input induced by order $\sigma$ is
$\bm{u}_\sigma = \bm{u}^{(N)}$.
Different orders can yield different control actions even from the same
nominal input, motivating the need for a principled structure over
admissible constraint orders.

In \PoSafeNet{}, this operator is instantiated per head by applying it to
the corresponding nominal policy $\bm{u}_{\mathrm{nom}}^{(h)}$ associated
with a linear extension $\sigma^{(h)}$.

\subsection{Poset-Structured Safety Composition}


We adopt the partial order $(\mathcal{S}, \preceq)$ over safety constraints
introduced in \cref{section: problem_formulation}.
A total order $\sigma = (j_1,\dots,j_N)$ is called a \emph{linear extension}
of the poset $(\mathcal{S}, \preceq)$ if for any $i,j\in\mathcal{S}$,
$i \prec j$ implies that $i$ appears before $j$ in $\sigma$.
Each linear extension $\sigma$ therefore defines a concrete sequential
safety enforcement order.

\begin{corollary}[Poset consistency of linear extensions]
\label{cor:linear_extension_poset}
Let $(\mathcal S,\preceq)$ be a safety poset.
For any linear extension $\sigma$ of the safety poset $(\mathcal S,\preceq)$, sequential enforcement of constraints as in~(\ref{eqn:sq}) according to
$\sigma$ satisfies poset-respecting safety in the sense of
Definition~\ref{def: poset-respecting safety}.
\end{corollary}

\begin{proof}
By definition of a linear extension, if $i \prec j$ then $i$ appears before
$j$ in $\sigma$. Thus, constraint $i$ can only be overridden when enforcing a strictly
higher-priority constraint $j$, which matches the definition of
poset-respecting safety.
\end{proof}

Corollary~\ref{cor:linear_extension_poset} establishes that each individual
linear extension corresponds to a valid safety execution, forming the basis
for representing multiple admissible safety compositions within
\PoSafeNet{}.


\begin{theorem}[Poset safety under sequential projection and antichain mixing]
\label{thm:general_poset_mixing}
Let $(\mathcal S,\preceq)$ be a safety poset.
Each constraint $i\in\mathcal S$ induces a closed convex feasible set
$H_i(\bm{x})\subset\mathbb{R}^m$ in control space as introduced in~(\ref{eqn:constraint_halfspace}).
Consider a policy that enforces constraints via sequential projection
as in~(\ref{eqn:sq}), according to any linear extension $\sigma$ of the
poset $(\mathcal S,\preceq)$, and combines multiple such executions only
through permutations within antichains.
Assume that for any pair of incomparable constraints $i\parallel j$,
projection onto one constraint does not violate feasibility of the other
(see Appendix~\ref{app:geometry} for sufficient geometric conditions).
Then the resulting policy satisfies poset-respecting safety in the sense
of Definition~\ref{def: poset-respecting safety}.
\end{theorem}

\paragraph{Geometric Intuition.}
For two incomparable constraints, the sufficient condition requires the projecting of one constraint does not move the control outside the feasible region of the other. Geometrically, the control halfspaces must intersect robustly, avoiding acute corners where projection onto one necessarily violates the other.

In our experiments, the navigation and manipulation tasks are designed to satisfy these sufficient conditions and therefore admit formal guarantees, whereas the vision-based driving task may violate them; in that case, safety guarantees should be interpreted as empirical.

The formal proof and sufficient geometric conditions for
Theorem~\ref{thm:general_poset_mixing} are provided in
Appendix~\ref{app:geometry}.

\subsection{\PoSafeNet{} Architecture}

\PoSafeNet{} is instantiated as a multi-head neural network.
Given state $\bm{x}$ and observations $\bm{z}$, the network produces
$H$ nominal control inputs $\{\bm{u}_{\mathrm{nom}}^{(h)}\}_{h=1}^H$,
each associated with a fixed linear extension $\sigma^{(h)}$.
Each head applies the corresponding sequential safety projection:
\begin{equation}
    \bm{u}^{(h)} = \Pi_{\sigma^{(h)}}\!\left(\bm{u}_{\mathrm{nom}}^{(h)}\right).
\end{equation}
All heads operate under a shared poset-structured safety specification and differ in their admissible execution orders, while safety semantics are fixed by a single shared poset.

\paragraph{Head combination and selection.} 
\label{par:headcombine}
During training, admissible safety executions are combined using a convex
mixture with global coefficients:
\begin{equation}
\label{eqn:head_mixture}
    \bm{u}
    =
    \sum_{h=1}^H \alpha_h \bm{u}^{(h)},
    \qquad
    \alpha_h \ge 0,\ \sum_h \alpha_h = 1.
\end{equation}
Alternatively, a Gumbel-softmax mechanism samples a single head
differentiably during training.
At inference time, either a single head is selected or a fixed mixture is
applied, yielding a deterministic safe control input.

\paragraph{Choice of linear extensions.}
While the number of linear extensions of a poset can be exponential, \PoSafeNet{} instantiates a finite set of $H$ heads in practice.
We include canonical extensions that respect obvious semantic hierarchies (e.g., collision avoidance or joint limits before task objectives),
and sample additional extensions when necessary to cover diverse priority orderings.
Empirically, we observe diminishing returns beyond a small $H$, which we fix across experiments.

\subsection{Training and Inference}

\PoSafeNet{} is trained via imitation learning on demonstrations
$(\bm{x}, \bm{z}, \bm{u}^*) \sim \mathcal{D}$ by minimizing
\begin{equation}
    \mathcal{L}(\bm{\theta})
    =
    \mathbb{E}_{(\bm{x},\bm{z},\bm{u}^*)\sim\mathcal{D}}
    \big[
        \ell(\bm{u}^*, \bm{u}(\bm{x}, \bm{z}; \bm{\theta}))
    \big].
\end{equation}
For each head, gradients propagate through the differentiable projection
operators, enabling joint optimization of nominal policy parameters and
constraint modulation parameters.

\PoSafeNet{} instantiates a set of nominal control policies, each associated
with a fixed linear extension of the safety poset.
For a given state, each head produces a nominal control input, which is then
mapped to a safe control via sequential projection onto the corresponding
constraint-induced halfspaces.
During training, head outputs are combined either through convex mixing or
differentiable hard selection, while gradients propagate through the
projection operators to optimize the nominal policies.
The overall construction and inference procedure is summarized in
Algorithm~\ref{alg:posafenet}.


\paragraph{Safety preservation.}
Learning and stochasticity in \PoSafeNet{} are restricted to selection among structurally valid safety compositions.
Each head satisfies poset-respecting safety by construction, as it executes a single linear extension of the safety poset.
\paragraph{Safety remark.}
Hard head selection is always safe, since it executes one poset-consistent linear extension.
Convex mixtures preserve safety only under additional geometric assumptions (Appendix~\ref{app:geometry}), namely when all heads lie within a common intersection of halfspaces or when incomparable constraints are mutually non-conflicting.
Outside these conditions, convex mixture should be interpreted as an empirical performance heuristic rather than a formal safety guarantee.

As a result, \PoSafeNet{} preserves the specified safety priority semantics by construction, with formal guarantees for hard selection and conditional guarantees for mixture.

\begin{algorithm}[t]
\caption{Construction and training of \PoSafeNet{}}
\label{alg:posafenet}

\textbf{Input:}
$\dot{\bm{x}} = f(\bm{x}) + g(\bm{x})\bm{u}$;
CBF constraints $\{b_i(\bm{x}) \ge 0\}_{i\in\mathcal S}$ with poset $(\mathcal S,\preceq)$;
linear extensions $\{\sigma^{(h)}\}_{h=1}^H$.

\textbf{Output:}
safe control $\bm{u}(\bm{x},\bm{z})$.

\begin{algorithmic}[1]
\STATE Initialize nominal policies $\{\bm{u}^{(h)}_{\mathrm{nom}}\}_{h=1}^H$
\STATE Initialize mixture weights $\bm{\alpha}$ or Gumbel logits $\bm{\gamma}$
\STATE Construct control-feasible halfspaces $H_i(\bm{x})$
as in~(\ref{eqn:constraint_halfspace}) and projection operators $\Pi_i$
as in~(\ref{eqn:cf})

\FOR{$h = 1,\dots,H$}
    \STATE $\bm{u}^{(h)} \gets \bm{u}^{(h)}_{\mathrm{nom}}(\bm{x},\bm{z})$
    \FOR{$j \in \sigma^{(h)}$}
        \STATE $\bm{u}^{(h)} \gets \Pi_j(\bm{u}^{(h)};\bm{x})$
    \ENDFOR
\ENDFOR

\STATE Select or mix $\{\bm{u}^{(h)}\}$ using $\bm{\alpha}$ or Gumbel-softmax
as described in \cref{par:headcombine}, to obtain
$\bm{u}(\bm{x},\bm{z})$

\end{algorithmic}
\end{algorithm}

\section{Experiments}
\begin{figure}[t]
    \centering   \includegraphics[width=0.95\linewidth]{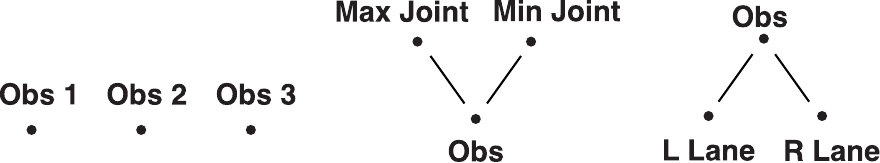}
    \caption{Task-specific safety posets used in the experiments.
    Left: unicycle navigation with mutually incomparable obstacle constraints.
    Middle: manipulator task where joint-limit constraints take precedence over obstacle avoidance.
    Right: vision-based driving where collision avoidance has higher priority than lane keeping.}
    \label{fig:task_posets}
\end{figure}

\begin{figure}[t]
    \centering
    \begin{subfigure}{0.48\columnwidth}
        \centering
        \includegraphics[width=\linewidth]{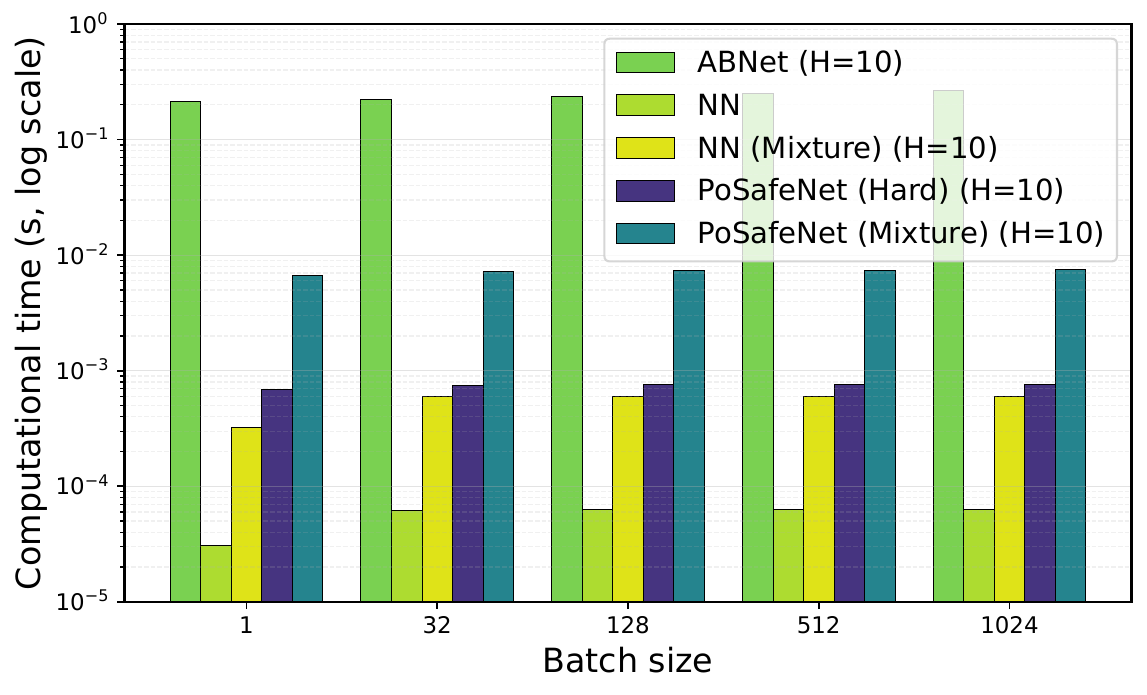}
        \caption{Inference (forward-only) time versus batch size with $H=10$ (log scale).}
        \label{fig:batchtime}
    \end{subfigure}\hfill
    \begin{subfigure}{0.48\columnwidth}
        \centering
        \includegraphics[width=\linewidth]{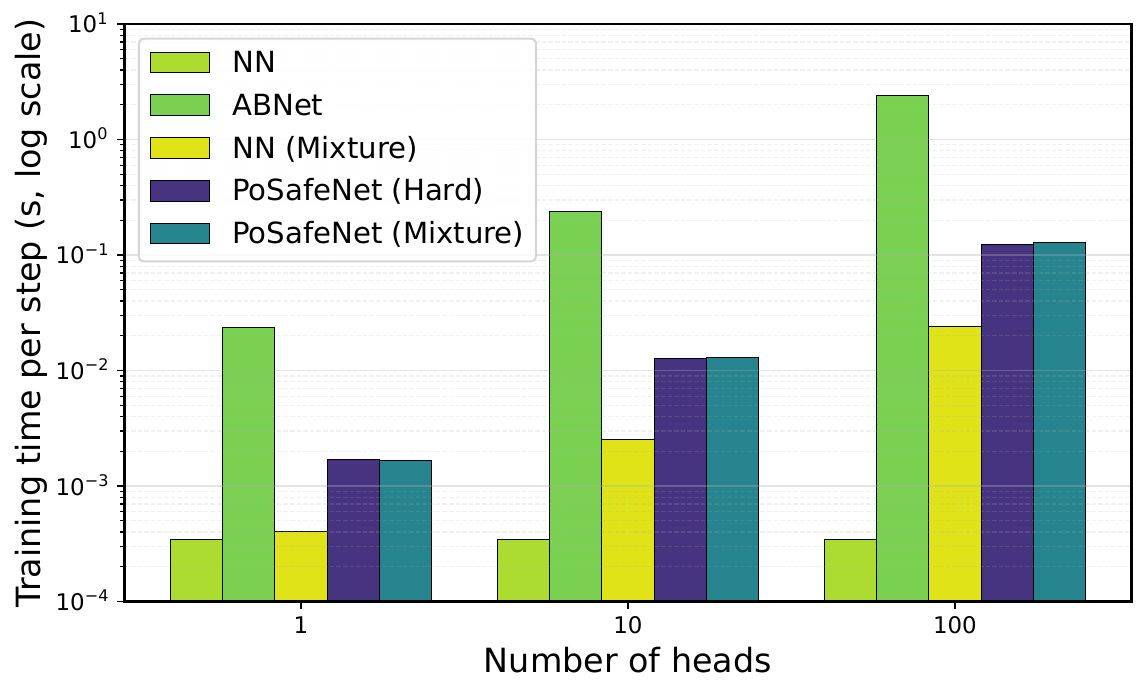}
        \caption{Training time per step versus number of heads with $B=128$ (log scale).}
        \label{fig:headtime}
    \end{subfigure}
    \caption{
    \textbf{Computational efficiency and scalability of \PoSafeNet{}.}
    (a) Inference cost versus batch size.
    (b) Training time per step versus number of heads.
    \PoSafeNet{} scales near-linearly in heads and is significantly faster than QP-based safety layers.
    }
    \label{fig:efficiency}
\end{figure}

In this section, we evaluate \PoSafeNet{} through a series of
experiments designed to answer the following questions:

\begin{itemize}
    \item How does \PoSafeNet{} compare with unstructured safety layers,
    including differentiable QP-based methods and prior neural barrier
    networks, in terms of computational efficiency and training
    scalability?

    \item Does enforcing safety through poset-structured sequential
    projection improve robustness and feasibility compared to
    simultaneous or slack-based enforcement, especially as the number
    of safety constraints increases?

    \item Can \PoSafeNet{} reliably enforce poset-respecting safety in scenarios with conflicting constraints, where not all
    safety objectives can be satisfied simultaneously?

    \item How does composing multiple poset-consistent safety policies, through head selection or convex mixture, affect safety, smoothness, and closed-loop performance?
\end{itemize}

\paragraph{Task-specific safety posets.}
Figure~\ref{fig:task_posets} visualizes the safety posets used in each task.
Nodes represent safety constraints and edges indicate strict priority relations.
These posets define the admissible safety compositions enforced by \PoSafeNet{}
in the following experiments.

\paragraph{Benchmark Models.} 
We compare \PoSafeNet{} against a diverse set of baseline and state-of-the-art
methods spanning end-to-end learning, unstructured safety enforcement, and
ensemble-based robustness.
(1) \textbf{End-to-end neural network (E2E)}, a plain neural network policy
trained end-to-end without explicit safety constraints \cite{levine2016end, 9812276};
(2) \textbf{Unstructured safety layers}, including BarrierNet \cite{xiao2023barriernet}, ABNet\cite{xiaoabnet2025}, and DFBNet\cite{pereira2021safe},
which enforce safety constraints without modeling priority, typically by
flattening all constraints into a single layer or introducing slack variables
that allow violations of higher-priority constraints; and
(3) \textbf{Structured safety (ours)}, which enforces safety through
poset-structured sequential projection, producing control outputs via
priority-consistent head selection or convex mixture.


\paragraph{Experimental protocol.}
Safety (min / mean) reports the minimum and average task-specific barrier value over the rollout. Values are normalized per task and are comparable only within each table, not across tasks. Detailed descriptions of datasets, expert demonstrations, rollout horizons, success criteria, safety margin definitions, and timing measurements are provided in Appendix~\ref{app:experiment}. All methods are evaluated under identical experimental settings and training protocols.

\subsection{Computation and Training Time}

We compare the computational efficiency of the proposed \PoSafeNet{} with
ABNet~\cite{xiaoabnet2025}, which enforces safety through a differentiable
quadratic programming (\ac{dQP}) layer~\cite{amos2017optnet}, as well as with
standard neural network baselines.
All timing benchmarks are conducted with
three simultaneously active CBF constraints, reflecting the target
multi-hazard scenarios considered in this work.
While ABNet admits a closed-form projection when enforcing at most two
linear constraints, this closed-form solution does not generalize to the
three-constraint setting evaluated here.
Therefore, for a fair and general comparison under multi-constraint
settings, we evaluate ABNet using its \ac{dQP}-based implementation.
The \ac{dQP} layer is implemented using \texttt{qpth} from OptNet.

Figure~\ref{fig:batchtime} reports inference time under varying batch sizes
with a fixed number of heads ($H=10$),
while Figure~\ref{fig:headtime} shows training time per optimization step
as the number of heads increases with a fixed batch size of 128.
\PoSafeNet{} exhibits computational costs comparable to standard neural
networks up to a small constant factor, while significantly outperforming
ABNet, whose QP-based safety layer incurs substantially higher overhead.

All timing experiments are conducted on a single NVIDIA GPU using PyTorch,
with QP layers implemented via \texttt{qpth}.

\subsection{2D Robot Multi-Obstacle Avoidance}

\begin{table*}[t]
\centering
\caption{
\textbf{Benchmark results on 2D unicycle obstacle avoidance (three circular obstacles).}
Results are reported over 100 stochastic rollouts (24 test episodes with repeated noise; see Appendix~\ref{app:experiment}).
\textit{QP Succ.\%} denotes the fraction of rollouts where the QP-based safety layer remains feasible.
\textit{Safety (min / mean)} reports the minimum and mean obstacle-safety margin over the rollout (negative indicates violation).
\textit{MSE} is computed w.r.t.\ the reference trajectory.
\textit{Feasibility} indicates whether a control input is produced at all steps.
\textit{Top Safety Guarantee} indicates \textit{Safety} $\ge 0$ throughout the rollout.
Values below $10^{-11}$ are treated as zero.
}
\label{tab:benchmark_unicycle_compact}
\resizebox{\textwidth}{!}{%
\begin{tabular}{lcccccccc}
\toprule
Method
& QP Succ.\% $\uparrow$
& Safety (min / mean) $\geq0,\downarrow$
& MSE (mean / var) $\downarrow$
& FinalDist$_{\mathrm{mean}}$ $\downarrow$
& Rollout Time (avg.)(s) $\downarrow$
& Unc.\ (u1,u2)
& Feasibility
& Top Safety Guarantee \\
\midrule

E2E \cite{levine2016end}
& 100
& $-81.00$ / $3.24$
& $5.72\!\times\!10^{2}$ / $1.26\!\times\!10^{7}$
& 26.57
& 0.028
& (4.87,\ 20.98)
& \checkmark
& $\times$ \\

DFB (slack=0) \cite{pereira2021safe}
& 78
& $-80.78$ / $48.68$
& $8.50\!\times\!10^{-2}$ / $3.46\!\times\!10^{-2}$
& 19.43
& 5.52
& (0.251,\ 0.313)
& $\times$
& $\times$ \\

DFB (slack=$10^3$) \cite{pereira2021safe}
& 100
& $0.144$ / $44.66$
& $2.47\!\times\!10^{-2}$ / $4.05\!\times\!10^{-4}$
& 15.40
& 6.57
& (0.146,\ 0.167)
& \checkmark
& \checkmark \\

BarrierNet (slack=0) \cite{xiao2023barriernet}
& 44
& $-43.99$ / $16.94$
& $1.96\!\times\!10^{-2}$ / $3.82\!\times\!10^{-4}$
& 7.81
& 7.09
& (0.136,\ 0.113)
& $\times$
& $\times$ \\

BarrierNet (slack=$10^3$) \cite{xiao2023barriernet}
& 83
& $-1.90\!\times\!10^{-5}$ / $32.59$
& $3.27\!\times\!10^{-2}$ / $1.37\!\times\!10^{-3}$
& 14.05
& 7.54
& (0.162,\ 0.198)
& $\times$
& $\times$ \\

ABNet (slack=0) \cite{xiaoabnet2025}
& 35
& $-61.28$ / $44.41$
& $9.45\!\times\!10^{-2}$ / $7.04\!\times\!10^{-2}$
& 13.14
& 36.37
& (0.102,\ 0.383)
& $\times$
& $\times$ \\

ABNet (slack=$10^3$)$^{\dagger}$ \cite{xiaoabnet2025}
& ---
& ---
& ---
& ---
& ---
& ---
& ---
& --- \\

PoSafeNet (Mixture) 
& 100
& $1.63$ / $37.20$
& $2.92\!\times\!10^{-2}$ / $1.61\!\times\!10^{-3}$
& 12.41
& 1.51
& (0.152,\ 0.185)
& \checkmark
& \checkmark \\

PoSafeNet (Hard)
& 100
& \textcolor{red}{$0.00$/ $31.63$}
& $2.03\!\times\!10^{-2}$ / $4.36\!\times\!10^{-4}$
& 10.95
& 0.276
& (0.132,\ 0.132)
& \checkmark
& \checkmark \\
\bottomrule
\end{tabular}
}

\vspace{2pt}
\footnotesize{$^{\dagger}$ ABNet (slack=0) failed to converge during training; metrics are not reported.}

\end{table*}
We consider a planar unicycle navigating toward a fixed goal while avoiding three circular obstacles.
All methods are evaluated on 24 held-out test episodes with repeated stochastic rollouts under observation noise (100 trials total).

The three obstacle-avoidance constraints are mutually incomparable and form an antichain, so any total ordering would impose arbitrary priorities.
QP-based safety layers cannot directly encode this structure and instead rely on simultaneous enforcement, which often leads to infeasibility, or on slack variables with manually chosen weights.
We report results for two representative slack settings.
In contrast, \PoSafeNet{} models the antichain by enumerating poset-consistent linear extensions and enforcing safety via sequential projection, without introducing artificial priority.

Table~\ref{tab:benchmark_unicycle_compact} summarizes the results.
QP-based methods without slack frequently fail due to infeasibility, while slack-based variants trade feasibility for conservative behavior or degraded safety margins.
\PoSafeNet{} achieves feasible rollouts with strictly positive safety margins across all trials, with hard head selection attaining a lower final distance to the goal.

Figure~\ref{fig:exp1_traj} shows representative trajectories.
QP-based baselines exhibit conservative detours or safety violations, whereas \PoSafeNet{} produces smooth, goal-directed trajectories that remain close to the reference.

\begin{figure}
    \centering
    \includegraphics[width=0.7\linewidth]{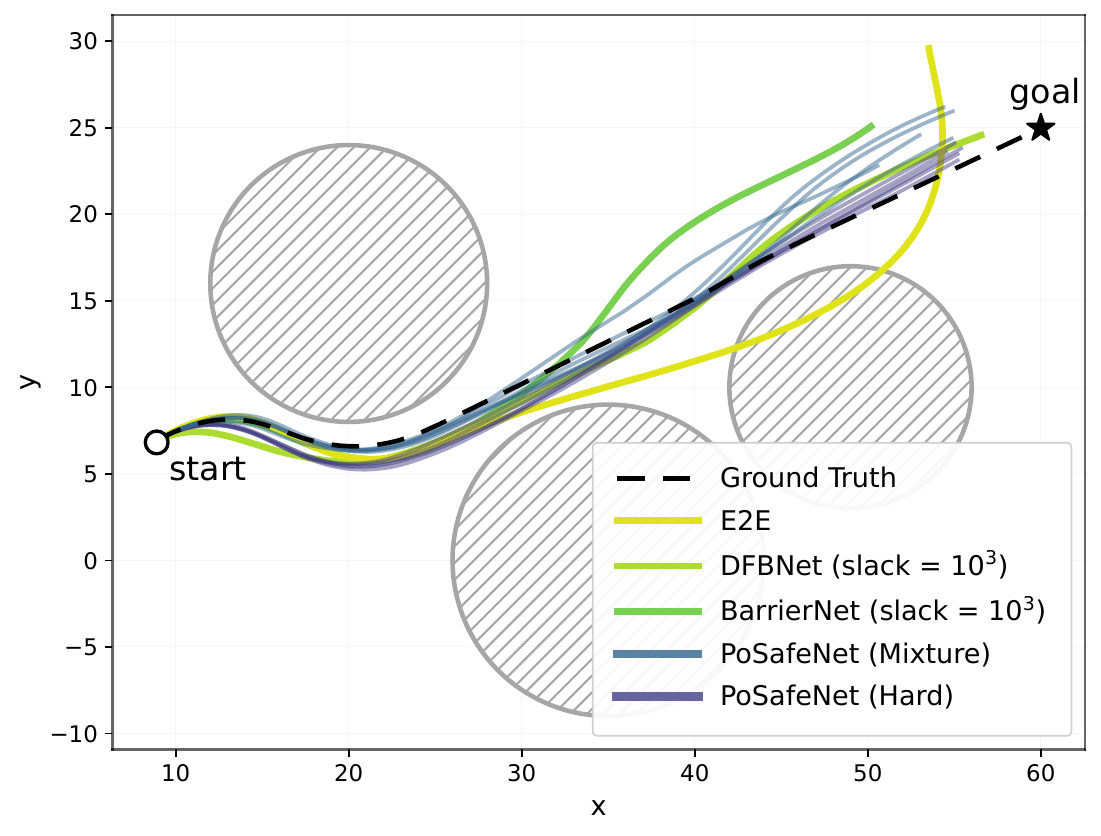}
    \caption{
    \textbf{Trajectories for 2D unicycle multi-obstacle avoidance.}
    Representative closed-loop trajectories in a planar navigation task with three circular obstacles in patches.
    The dashed black curve denotes the reference trajectory.
    Unstructured baselines exhibit either conservative detours or safety violations.
    In contrast, \PoSafeNet{} produces smooth, goal-directed trajectories while maintaining safety throughout the rollout.
    }
    \label{fig:exp1_traj}
\end{figure}

\subsection{Safe Robot Manipulation under Joint Constraints}
For the manipulator task, we evaluate safety-critical control on a two-link planar manipulator with joint limits and obstacle avoidance.
All methods are trained via imitation learning and evaluated on a held-out test set under identical protocols.
The objective is to track a reference end-effector trajectory while respecting joint-space safety constraints.

Joint-space constraints, captured by the $\phi$ safety measure, take precedence over end-effector safety, as violating joint limits is unacceptable even if the end-effector must enter forbidden regions.
This induces a non-trivial poset with strict priority relations, in contrast to the antichain structure in the navigation task.
QP-based safety layers enforce these heterogeneous constraints either simultaneously or via slack variables, which can lead to numerical instability or residual violations under tight joint limits.
In contrast, \PoSafeNet{} directly encodes this precedence relation and enforces safety through sequential projection.

Table~\ref{tab:benchmark_manipulator_compact} summarizes the results.
Unstructured baselines frequently exhibit joint-limit violations or instability, even with slack variables.
\PoSafeNet{} achieves zero $\phi$-violation across all successful episodes while maintaining low tracking error and fast rollout time.
Both hard and mixture variants satisfy all joint constraints, with hard selection yielding slightly lower computation cost and mixture yielding lower reference error.

Figure~\ref{fig:exp2_traj} shows representative end-effector trajectories.
Unstructured methods exhibit oscillations or violations near joint limits, whereas \PoSafeNet{} closely tracks the reference while respecting all constraints.

\begin{figure}
    \centering
    \includegraphics[width=0.7\linewidth]{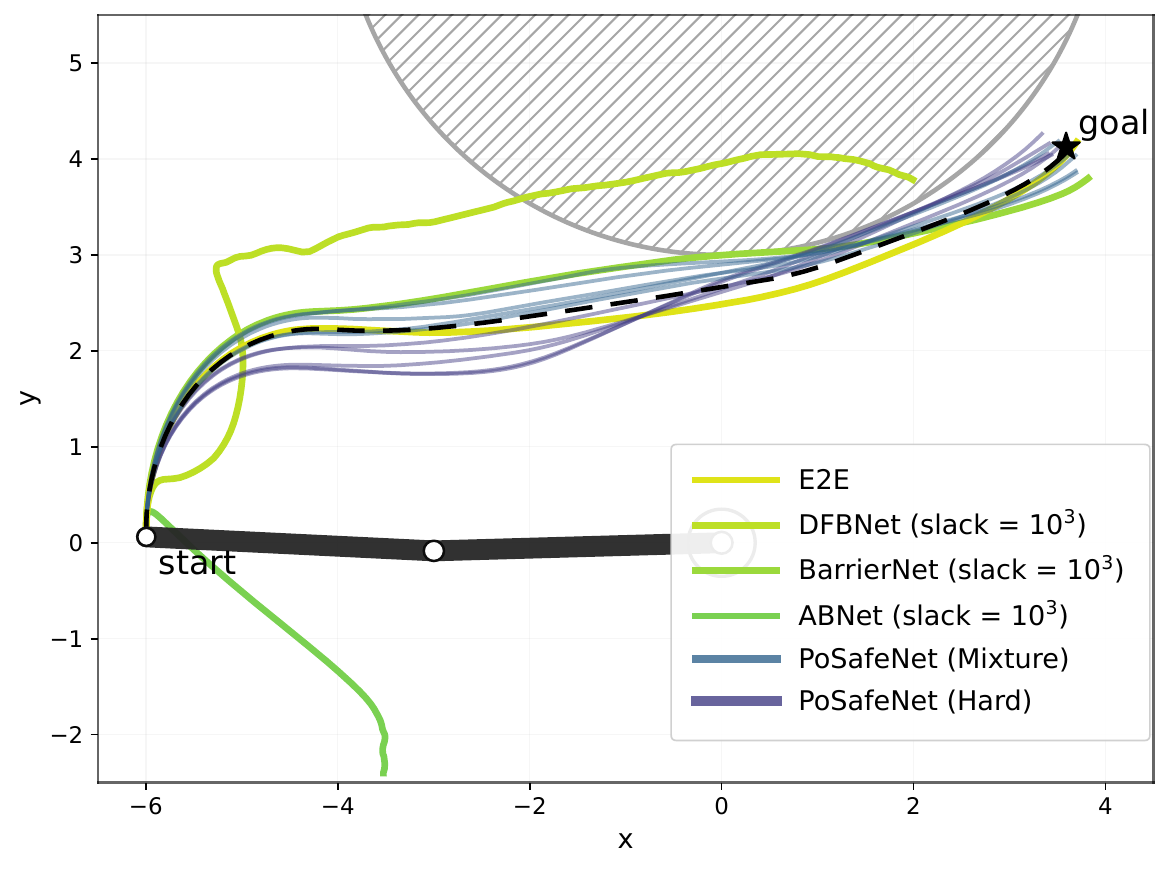}
    \caption{
    \textbf{Trajectories for safe robot manipulation under joint-space constraints.}
    Representative end-effector trajectories for a two-link planar manipulator.
    The dashed black curve denotes the reference trajectory.
    Unstructured methods exhibit oscillatory or unstable behavior near joint limits.
    \PoSafeNet{} closely tracks the reference while respecting joint-space safety throughout the rollout.
    }
    \label{fig:exp2_traj}
\end{figure}

\begin{table*}[t]
\centering
\caption{
\textbf{Benchmark results on the two-link manipulator under joint and obstacle constraints.}
Results are reported over the held-out test set (see Appendix~\ref{app:experiment}).
\textit{Safety (min / mean)} reports the minimum and mean obstacle-safety margin (negative indicates violation).
\textit{$\phi$ (mean / max)} quantifies joint-space violation (lower is better).
\textit{MSE} is computed w.r.t.\ the reference end-effector trajectory.
\textit{Feasibility} indicates whether a control input is produced at all steps.
\textit{Top Safety Guarantee} indicates zero $\phi$-violation throughout the rollout.
Values below numerical precision are treated as zero.
}

\label{tab:benchmark_manipulator_compact}
\resizebox{\textwidth}{!}{%
\begin{tabular}{lccccccccc}
\toprule
Method
& QP Succ.\% $\uparrow$
& Safety (min / mean) $\geq0,\downarrow$
& MSE (mean / var) $\downarrow$
& $\phi$ (mean / max)$\downarrow$
& FinalDist$_{\mathrm{mean}}\downarrow$
& Rollout Time (avg.) (s)$\downarrow$
& Feasibility
& Top Safety Guarantee \\
\midrule

E2E \cite{levine2016end}
& 100
& $-12.04$ / $2.30$
& $9.64\!\times\!10^{-3}$ / $1.06\!\times\!10^{-4}$
& $2.53\!\times\!10^{-2}$ / $1.98\!\times\!10^{-1}$
& $0.702$
& $0.021$
& $\checkmark$
& $\times$ \\

DFBNet (slack=0) \cite{pereira2021safe}
& 100
& $-14.99$ / $-10.74$
& $4.04\!\times\!10^{-1}$ / $3.95\!\times\!10^{-2}$
& $4.05\!\times\!10^{-2}$ / $3.26\!\times\!10^{-1}$
& $2.53$
& $1.62$
& $\checkmark$
& $\times$ \\

DFBNet (slack=$10^3$) \cite{pereira2021safe}
& 100
& $-14.84$ / $-8.08$
& $3.86\!\times\!10^{-1}$ / $3.48\!\times\!10^{-2}$
& $4.99\!\times\!10^{-6}$ / $4.56\!\times\!10^{-5}$
& $2.29$
& $1.71$
& $\checkmark$
& $\times$ \\

BarrierNet (slack=0) \cite{xiao2023barriernet}
& 100
& $-9.63$ / $2.23$
& $1.33\!\times\!10^{-2}$ / $1.39\!\times\!10^{-4}$
& $3.29\!\times\!10^{-2}$ / $2.51\!\times\!10^{-1}$
& $0.672$
& $1.59$
& $\checkmark$
& $\times$ \\

BarrierNet (slack=$10^3$) \cite{xiao2023barriernet}
& 100
& $-4.07$ / $2.21$
& $1.21\!\times\!10^{-2}$ / $1.39\!\times\!10^{-4}$
& \textcolor{red}{$0.0000$ / $0.0000$}
& $0.410$
& $1.60$
& $\checkmark$
& $\checkmark$ \\

ABNet (slack=0) \cite{xiaoabnet2025}
& 76
& $-14.35$ / $43.13$
& $1.46\!\times\!10^{116}$ / $1.60\!\times\!10^{234}$
& $0.417$ / $0.838$
& NaN
& $0.550$
& $\times$
& $\times$ \\

ABNet (slack=$10^3$) \cite{xiaoabnet2025}
& 84
& $31.03$ / $62.05$
& $4.20$ / $1.48\!\times\!10^{2}$
& $0.167$ / $0.837$
& $9.48$
& $0.544$
& $\times$
& $\times$ \\

PoSafeNet (Mixture)
& 100
& $0.199$ / $2.44$
& $9.08\!\times\!10^{-3}$ / $1.12\!\times\!10^{-4}$
& \textcolor{red}{$0.0000$ / $0.0000$}
& $0.361$
& $0.173$
& $\checkmark$
& $\checkmark$ \\

PoSafeNet (Hard)
& 100
& \textcolor{red}{$0.018$ / $1.38$}
& $1.04\!\times\!10^{-2}$ / $1.27\!\times\!10^{-4}$
& \textcolor{red}{$0.0000$ / $0.0000$}
& $0.527$
& $0.095$
& $\checkmark$
& $\checkmark$ \\

\bottomrule
\end{tabular}
}
\end{table*}

\subsection{Vision-Based End-to-End Autonomous Driving}

\begin{table*}[t]
\centering
\caption{
\textbf{Benchmark results on VISTA vision-based autonomous driving.}
Results are reported over the held-out test set under identical protocols
(see Appendix~\ref{app:experiment}).
\textit{Safety (min / mean)} reports barrier values (negative indicates violation).
\textit{Lane viol. (min / mean)} reports lane margin (negative indicates departure).
\textit{$|e_y|$ (mean / var)} quantifies lateral deviation from the lane centerline.
\textit{Time-out-of-lane} denotes the fraction of time outside lane boundaries.
\textit{Pass\%} indicates rollouts completed without collision.
\textit{Rot viol.} measures mean rotation rate magnitude.
\textit{Uncertainty (u1,u2)} denotes control input standard deviation.
Higher is better for Pass and Safety; lower is better for all other metrics.
}

\label{tab:benchmark_vista_compact}
\resizebox{\textwidth}{!}{%
\begin{tabular}{lcccccccc}
\toprule
Method
& Pass\% $\uparrow$
& Crash\% $\downarrow$
& Safety (min / mean) $\geq0,\downarrow$
& Lane viol.\ (min / mean) $\downarrow$
& $|e_y|$ (mean / var) $\downarrow$
& Time-out-of-lane $\downarrow$
& Rot viol.\ $\downarrow$
& Unc.\ (u1,u2) \\
\midrule
 
E2E \cite{9812276}
& 15.0
& 85.0
& $-33.63$ / $-22.39$
& $-1.90$ / $0.87$
& $0.396$ / $0.330$
& 0.0187
& 0.07
& $(0.134,\;0.413)$ \\

DFBNet (w/ lane constraint) \cite{pereira2021safe}
& 100.0
& 0.0
& $12.36$ / $19.07$
& $-2.62$ / $-1.57$
& $1.509$ / $2.745$
& 0.3769
& 0.52
& $(0.290,\;0.542)$ \\

DFBNet (w/o lane constraint) \cite{pereira2021safe}
& 100.0
& 0.0
& $13.11$ / $19.73$
& $-2.78$ / $-1.65$
& $1.514$ / $2.871$
& 0.3786
& 0.61
& $(0.288,\;0.548)$ \\

BarrierNet (w/ lane constraint) \cite{xiao2023barriernet}
& 0.0
& 100.0
& $-44.87$ / $-31.76$
& $0.07$ / $0.09$
& $0.939$ / $0.607$
& 0.0000
& 0.00
& $(0.038,\;0.406)$ \\

BarrierNet (w/o lane constraint) \cite{xiao2023barriernet}
& 100.0
& 0.0
& $-0.17$ / $5.82$
& $-4.23$ / $-1.57$
& $1.364$ / $2.522$
& 0.2886
& 0.73
& $(0.384,\;0.558)$ \\

ABNet (w/ lane constraint) \cite{xiaoabnet2025}
& 25.0
& 75.0
& $-46.16$ / $-23.60$
& $-0.03$ / $0.18$
& $1.174$ / $0.433$
& 0.0662
& 0.00
& $(0.099,\;0.651)$ \\

ABNet (w/o lane constraint) \cite{xiaoabnet2025}
& 100.0
& 0.0
& $5.65$ / $8.33$
& $-2.24$ / $-1.17$
& $1.232$ / $2.154$
& 0.3141
& 0.80
& $(0.401,\;0.423)$ \\

PoSafeNet (Mixture)
& 100.0
& 0.0
& \textcolor{red}{$1.23$} / $8.21$
& $-0.93$ / $-0.49$
& $1.028$ / $1.278$
& 0.1992
& 0.14
& $(0.287,\;0.387)$ \\

PoSafeNet (Hard)
& 100.0
& 0.0
& $2.73$ / \textcolor{red}{$7.12$}
& $-0.63$ / $-0.32$
& $1.404$ / $1.599$
& 0.3670
& 0.80
& $(0.345,\;0.456)$ \\

\bottomrule
\end{tabular}}
\end{table*}

We evaluate \PoSafeNet{} on a vision-based autonomous driving task in the \texttt{VISTA} simulator,
where policies map raw visual observations directly to control actions.
All methods are trained via imitation learning and evaluated on a held-out test set under identical protocols.
Safety constraints include collision avoidance and lane keeping, with the objective of driving forward smoothly without crashes.

Collision avoidance takes precedence over lane keeping, allowing lower-priority constraints to be temporarily violated to preserve higher-priority safety.
QP-based and unstructured safety layers enforce these constraints either simultaneously or via slack variables, which can lead to brittle behavior under visual uncertainty.
For comparability, we also report BarrierNet and ABNet without lane constraints.

Table~\ref{tab:benchmark_vista_compact} summarizes the results.
Unstructured baselines with all constraints imposed either suffer frequent collisions or exhibit unstable behavior near obstacles.
Although BarrierNet without lane constraints achieves a lower mean $|e_y|$ in some cases, it lacks an explicit priority structure for resolving heterogeneous safety objectives, resulting in higher variance and less reliable behavior.
\PoSafeNet{} trades slight increases in mean deviation for more consistent adherence to safety priorities, achieving a $100\%$ pass rate with zero collisions while maintaining low lane violation and smooth control.
The hard and mixture variants attain comparable safety performance, with the mixture producing smoother steering, lower lane violation, and reduced control variance.

Figure~\ref{fig:exp3_traj} shows representative trajectories.
Unstructured methods frequently leave the lane or over-correct steering, whereas \PoSafeNet{} maintains lane adherence and collision avoidance throughout the rollout.
The mixture variant exhibits improved lane-center adherence, while the hard variant commits to one side of the lane, reflecting selection of a single poset-consistent order.
{\it Additional ablations are provided in Appendix~\ref{app:ablation}.}

\begin{figure}
    \centering
    \includegraphics[width=0.9\linewidth]{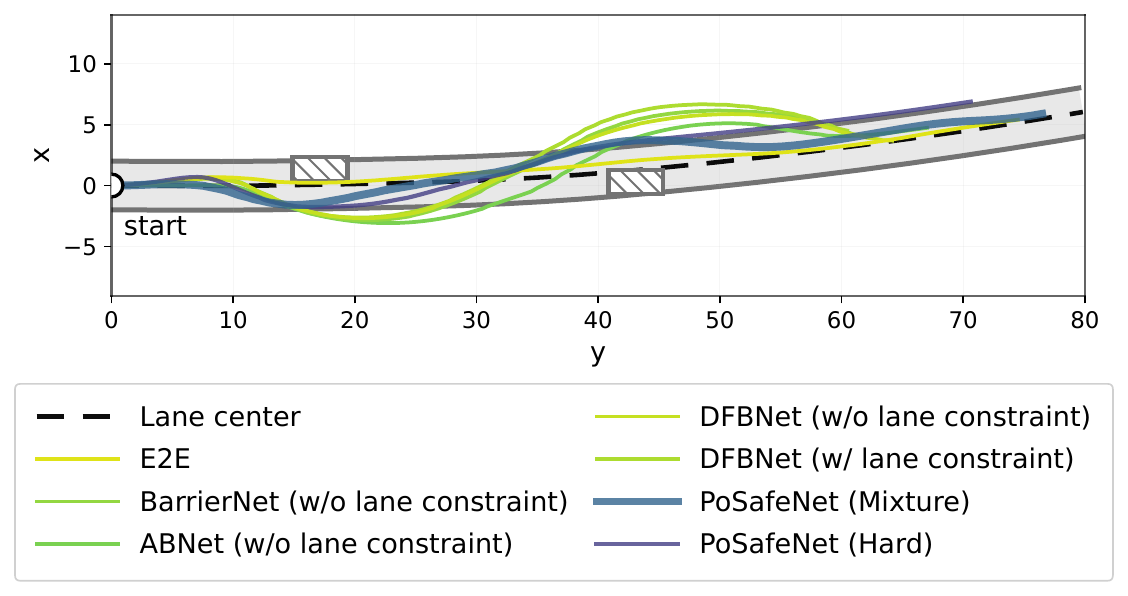}
    \caption{
    \textbf{Trajectories for vision-based autonomous driving.}
    Representative closed-loop driving trajectories in the \texttt{VISTA} simulator.
    Unstructured baselines exhibit unstable behavior near obstacles or lane departures.
    \PoSafeNet{} maintains collision avoidance and lane adherence throughout the rollout,
    consistent with the imposed priority of collision avoidance over lane keeping.
    }
    \label{fig:exp3_traj}
\end{figure}
\section{Conclusions, Limitations and Future Work}

We propose \PoSafeNet{}, a learning-based control framework that enforces safety through
explicit modeling of partial orders over safety constraints, enabling systematic relaxation
when simultaneous enforcement is infeasible.
By treating safety composition as a structural property of the policy class,
\PoSafeNet{} enables priority-aware and feasibility-preserving enforcement of multiple
control barrier functions via differentiable sequential projection.
Across diverse robotic tasks—including multi-obstacle navigation, constrained robot manipulation,
and vision-based autonomous driving—we demonstrate that respecting poset-structured safety
leads to improved feasibility, robustness, and interpretability compared to unstructured
or slack-based safety layers.

\paragraph{Limitations.}
First, \PoSafeNet{} assumes that a meaningful partial order over safety constraints
can be specified or inferred from task semantics; designing such a \ac{poset}
may require domain knowledge. Second, the constraints in \PoSafeNet{} can be designed by prior knowledge in structured environments, but they are not available in unstructured environments.
Third, our current formulation focuses on control-affine systems and barrier-based safety certificates,
and does not directly address stochastic dynamics or probabilistic safety specifications.

\paragraph{Future Work.}
Future work will explore automated discovery and learning of safety \ac{poset}s from data,
as well as scalable approximations for handling large or dynamic partial orders.
Another promising direction is extending \PoSafeNet{} to stochastic and multi-agent settings,
where safety priorities may evolve over time or depend on interactions with other agents.
Incorporating probabilistic safety guarantees and tighter integration with perception uncertainty
also represents an important avenue for deploying poset-respecting safety
in real-world robotic systems.

\bibliographystyle{plainnat}
\bibliography{reference}

\clearpage
\appendix
\onecolumn

\section{Control Barrier Function}
\label{app:CBF}
This appendix clarifies how \ac{CBF} induce affine halfspace constraints in the control input space and provides a closed-form solution to the optimization problem involving a single \ac{CBF} constraint.

\subsection{From CBF Conditions to Control Halfspaces}

Consider a control-affine system
\[
\dot{\bm{x}} = f(\bm{x}) + g(\bm{x})\,\bm{u}.
\]
Given a continuously differentiable control barrier function $b_j(\bm{x})$,
the standard CBF condition
\[
\dot b_j(\bm{x},\bm{u}) + \alpha\big(b_j(\bm{x})\big) \ge 0
\]
induces an affine constraint on the control input.
Specifically, enforcing the above condition yields
\begin{equation}
A_j(\bm{x})\,\bm{u} \;\ge\; c_j(\bm{x}),
\label{eq:cbf_halfspace}
\end{equation}
where
\[
A_j(\bm{x}) := L_g b_j(\bm{x}) \in \mathbb{R}^{1\times m},
\qquad
c_j(\bm{x}) := -L_f b_j(\bm{x}) - \alpha\big(b_j(\bm{x})\big) \in \mathbb{R}.
\]
Although safety constraints may be nonlinear in the state space,
for control-affine systems they always induce affine halfspace constraints
in the control input space.

\subsection{Safety Constraints as Halfspace Projections}

The affine inequality~\eqref{eq:cbf_halfspace} defines a halfspace in the
control input space.
Safety can therefore be enforced by projecting a nominal control input
$\bm{u}_{\mathrm{nom}}$ onto the feasible halfspace associated with constraint $j$.

Specifically, for a candidate control input $\bm{u}\in\mathbb{R}^m$,
we define the projection operator
\begin{equation}
\Pi_j(\bm{u})
=
\arg\min_{\bm{v}\in\mathbb{R}^m}
\;\frac12\|\bm{v}-\bm{u}\|^2
\quad
\text{s.t.}
\quad
A_j(\bm{x})\,\bm{v} \ge c_j(\bm{x}).
\label{eq:halfspace_projection}
\end{equation}
\paragraph{KKT derivation.}
For notational simplicity, suppress the state dependence and write
$A:=A_j(\bm{x})$ and $c:=c_j(\bm{x})$.
Consider the projection problem
\[
\min_{\bm{u}\in\mathbb{R}^m}
\;\frac12\|\bm{u}-\bm{u}_0\|^2
\quad
\text{s.t.}
\quad
A\bm{u}\ge c,
\]
where $\bm{u}_0$ denotes a given nominal control input.

The Lagrangian is
\[
\mathcal{L}(\bm{u},\lambda)
=
\frac12\|\bm{u}-\bm{u}_0\|^2
+
\lambda\,(c-A\bm{u}),
\qquad \lambda\ge 0.
\]

The Karush--Kuhn--Tucker (KKT) conditions are:
\begin{align*}
\textbf{Stationarity:} \quad &
(\bm{u}-\bm{u}_0)-\lambda A^\top = \bm{0}
\;\Rightarrow\;
\bm{u}=\bm{u}_0+\lambda A^\top,\\
\textbf{Primal feasibility:} \quad &
A\bm{u}\ge c,\\
\textbf{Dual feasibility:} \quad &
\lambda\ge 0,\\
\textbf{Complementary slackness:} \quad &
\lambda\,(c-A\bm{u})=0.
\end{align*}

Substituting $\bm{u}=\bm{u}_0+\lambda A^\top$ into the constraint yields
\[
c-A(\bm{u}_0+\lambda A^\top)
=
(c-A\bm{u}_0)-\lambda \|A\|^2,
\]
where $\|A\|^2 := AA^\top$ for $A\in\mathbb{R}^{1\times m}$.
If $A\bm{u}_0\ge c$, complementary slackness implies $\lambda=0$ and the solution
is $\bm{u}^\star=\bm{u}_0$.
Otherwise, the constraint is active and
\[
\lambda^\star = \frac{c-A\bm{u}_0}{\|A\|^2}.
\]

\paragraph{Closed-form solution.}
Combining both cases, the projection admits the closed-form expression
\begin{equation}
\Pi_j(\bm{u})
=
\bm{u}
+
\frac{\mathrm{ReLU}\!\big(c_j(\bm{x}) - A_j(\bm{x})\,\bm{u}\big)}
{\|A_j(\bm{x})\|^2}
\,A_j(\bm{x})^\top .
\label{eq:halfspace_projection_relu}
\end{equation}

The ReLU term explicitly activates only when the constraint is violated,
ensuring that feasible control inputs remain unchanged while infeasible
inputs are corrected by an orthogonal projection onto the constraint boundary.
This closed-form operator serves as the basic building block for the
sequential projection and mixing framework developed in the main text.

\section{Geometric Conditions for Safe Sequential Projection and Mixing}
\label{app:geometry}

This appendix characterizes geometric conditions under which sequential enforcement
and convex mixing of multiple \ac{CBF} constraints preserve feasibility.
We also explain why these conditions are frequently satisfied in robotic systems,
and how their violation motivates poset-structured safety enforcement.

\subsection{CBF Constraints as Control Halfspaces}

Consider a control-affine system
\[
\dot{\bm{x}} = f(\bm{x}) + g(\bm{x})\,\bm{u},
\]
and a collection of \ac{CBF} $b_i(\bm{x})$.
The standard constraint condition
\[
\dot b_i(\bm{x},\bm{u}) + \alpha\big(b_i(\bm{x})\big) \ge 0
\]
induces an affine constraint on the control input of the form
\[
A_i(\bm{x})\,\bm{u} \ge c_i(\bm{x}),
\qquad
A_i(\bm{x}) := L_g b_i(\bm{x}) = \nabla b_i(\bm{x})^\top g(\bm{x}),
\]
which defines a halfspace in the control input space.
The feasible control set associated with multiple \ac{CBF} constraints is therefore
the intersection of such halfspaces.

\subsection{Two Sources of Compatibility Between CBF Constraints}

Empirically, sequential projection and convex mixing of \ac{CBF} constraints often
preserve feasibility in robotic systems, though this behavior is not universal.
We identify two geometric conditions—acting in different spaces—that explain when
such compatibility is expected.

\subsubsection{Condition C1: Aligned repulsive directions in configuration space}

Each \ac{CBF} encodes a repulsive constraint with respect to a hazard in the
configuration space $\bm{q}$.
For common geometric barriers such as obstacle avoidance, collision radii, or lane
boundaries, the gradient $\nabla_{\bm{q}} b_i(\bm{x})$ points in the direction that
locally increases the safety margin.

We assume that, at a given configuration,
\[
\nabla_{\bm{q}} b_i(\bm{x})^\top \nabla_{\bm{q}} b_j(\bm{x}) \ge 0,
\qquad i \neq j,
\]
that is, the repulsive directions associated with different hazards are not opposing.
Geometrically, this corresponds to hazards whose avoidance induces compatible motions
(e.g., multiple constraints requiring deceleration or steering in the same general
direction).

\subsubsection{Condition C2: Control-channel isotropy}

Let $g_{\bm{q}}(\bm{x})$ denote the block of $g(\bm{x})$ mapping control inputs to
configuration dynamics $\dot{\bm{q}}$.
We assume that locally
\[
g_{\bm{q}}(\bm{x}) g_{\bm{q}}(\bm{x})^\top = \gamma(\bm{x}) I,
\qquad \gamma(\bm{x}) > 0,
\]
or more generally that $g_{\bm{q}} g_{\bm{q}}^\top$ approximately preserves angles in
configuration space.
This condition holds exactly for single-integrator models and approximately for many
fully actuated or feedback-linearized robotic systems.

\subsection{Implication for Control-Space Halfspace Normals}

Under Conditions~C1--C2, the control-space halfspace normals inherit the alignment of
the underlying configuration-space repulsive directions.

\begin{lemma}[Compatibility of \ac{CBF} halfspaces in control space]
\label{lem:noninterference}
Under Conditions~C1--C2, the control-space normals satisfy
\[
A_i(\bm{x})^\top A_j(\bm{x}) \ge 0,
\qquad \forall\, i \neq j .
\]
\end{lemma}

\begin{proof}
Restricting attention to the configuration component, we have
$A_i(\bm{x}) = \nabla_{\bm{q}} b_i(\bm{x})^\top g_{\bm{q}}(\bm{x})$.
Thus, for any $i \neq j$,
\[
\begin{aligned}
A_i(\bm{x})^\top A_j(\bm{x})
&= \nabla_{\bm{q}} b_i(\bm{x})^\top
   g_{\bm{q}}(\bm{x}) g_{\bm{q}}(\bm{x})^\top
   \nabla_{\bm{q}} b_j(\bm{x}) \\
&= \gamma(\bm{x})\,
   \nabla_{\bm{q}} b_i(\bm{x})^\top
   \nabla_{\bm{q}} b_j(\bm{x}),
\end{aligned}
\]
where the second equality follows from Condition~C2.
By Condition~C1 and $\gamma(\bm{x}) > 0$, the claim follows.
\end{proof}

\begin{theorem}[Poset safety under sequential projection and antichain mixing]
Let $(\mathcal S,\preceq)$ be a safety poset.
Each constraint $i\in\mathcal S$ induces a (closed) convex feasible set
$H_i\subset\mathbb{R}^m$ in control space; in particular, for CBF halfspaces,
$H_i=\{\bm{u}:\ A_i^\top \bm{u}\ge c_i\}$.

\paragraph{Enforcement mechanism.}
Fix any linear extension (topological order)
$\sigma=(j_1,\dots,j_{|\mathcal S|})$ of $\preceq$ such that lower-priority elements
appear earlier and higher-priority elements appear later:
$i\prec j \Rightarrow i \text{ appears before } j \text{ in }\sigma$.
Given a nominal control $u_{\mathrm{nom}}$, define the sequential enforcement iterates
\[
\bm{u}^{(0)}:=\bm{u}_{\mathrm{nom}},\qquad
\bm{u}^{(t)}:=\Pi_{j_t}\big(\bm{u}^{(t-1)}\big),\quad t=1,\dots,|\mathcal S|,
\]
where $\Pi_{j}$ denotes Euclidean projection onto $H_j$.
Let the final output be $u_{\mathrm{out}}:=\bm{u}^{(|\mathcal S|)}$.

\paragraph{Antichain mixing.}
Consider $K$ heads that share the same relative order on comparable constraints
(i.e., each head uses a linear extension of $\preceq$) and may differ only by
permutations within antichains.
Let $\bm{u}^{\mathrm{out}}_r$ be the output of head $r$ and define the mixed control
$\bar {\bm{u}}:=\sum_{r=1}^K \theta_r \bm{u}^{\mathrm{out}}_r$ with $\theta_r\ge 0$, $\sum_r\theta_r=1$.

\paragraph{Compatibility assumption (no override among incomparable constraints).}
Assume that for any incomparable pair $i\parallel j$ (neither $i\preceq j$ nor $j\preceq i$),
projection onto one constraint does not violate the other:
\begin{equation}
\bm{u}\in H_j \ \Rightarrow\ \Pi_i(\bm{u})\in H_j,
\qquad
\bm{u}\in H_i \ \Rightarrow\ \Pi_j(\bm{u})\in H_i.
\tag{NI}
\label{eq:NI}
\end{equation}
For halfspaces, a sufficient condition is $A_i^\top A_j\ge 0$ together with
$H_i\cap H_j\neq\emptyset$ (cf. Lemma~\ref{lem:noninterference}).

Then the following hold.

\begin{enumerate}
\item \textbf{(Poset-respecting override during enforcement).}
Along the sequential enforcement process, a constraint $i$ can become violated
(i.e., satisfied at some step and violated at a later step) only when enforcing
a constraint $j$ with strictly higher priority: $i\prec j$.
In particular, no incomparable constraint can override another.

\item \textbf{(Mixing preserves poset-respecting safety).}
Assume furthermore that for every constraint $j\in\mathcal S$,
each head output satisfies $\bm{u}^{\mathrm{out}}_r\in H_j$ whenever $j$ is maximal
with respect to $\preceq$ (i.e., it has no higher-priority successor).
Then the mixed control $\bar u$ also satisfies all maximal constraints,
and the mixed policy preserves poset-respecting safety under Definition~3.1
(with the override rule above): any violation of a lower-priority constraint
can only be attributed to enforcing some higher-priority constraint, and mixing
does not introduce violations among incomparable constraints.
\end{enumerate}
\end{theorem}

\begin{proof}
\textbf{(1)} Fix any time $t$ and suppose a constraint $i$ is satisfied at step $t-1$
but becomes violated at step $t$, i.e., $\bm{u}^{(t-1)}\in H_i$ and $\bm{u}^{(t)}\notin H_i$.
By definition $\bm{u}^{(t)}=\Pi_{j_t}(\bm{u}^{(t-1)})$.
If $i\parallel j_t$, then by the non-interference assumption~\eqref{eq:NI},
$\bm{u}^{(t-1)}\in H_i$ would imply $\Pi_{j_t}(\bm{u}^{(t-1)})\in H_i$, contradicting $\bm{u}^{(t)}\notin H_i$.
Hence $i$ is not incomparable with $j_t$.
Since $\sigma$ is a linear extension and $j_t$ is enforced at step $t$,
the only remaining possibility consistent with ``low-to-high'' enforcement is that
$j_t$ has higher priority than $i$, i.e., $i\prec j_t$.
Thus $i$ can be overridden only by enforcing a higher-priority constraint.

\textbf{(2)} Let $j$ be any maximal element of the poset.
By assumption, $\bm{u}_r^{\mathrm{out}}\in H_j$ for all heads $r$.
Because $H_j$ is convex, it is closed under convex combinations, hence
$\bar {\bm{u}}=\sum_r \theta_r \bm{u}_r^{\mathrm{out}}\in H_j$.
Therefore mixing preserves satisfaction of all maximal (highest-priority) constraints.
Together with part (1), this implies that (i) violations can only occur by higher-priority
overrides, and (ii) incomparable constraints are never made to override each other,
which is exactly poset-respecting safety under the stated override rule.
\end{proof}

\subsection{Limitations and the Role of Posets}

The above conditions are sufficient but not universal.
When multiple hazards induce opposing control corrections—for example, constraints
requiring steering in opposite directions or competing accelerations—Condition~C1
fails, yielding $A_i^\top A_j < 0$.
In this case, enforcing one CBF necessarily degrades another.

Such intrinsically conflicting constraints must therefore be ordered or aggregated,
rather than treated as an antichain.
The poset structure in our method reflects this geometric compatibility between safety
constraints, rather than being an arbitrary design choice.

\begin{figure}
    \centering
    \includegraphics[width=0.7\linewidth]{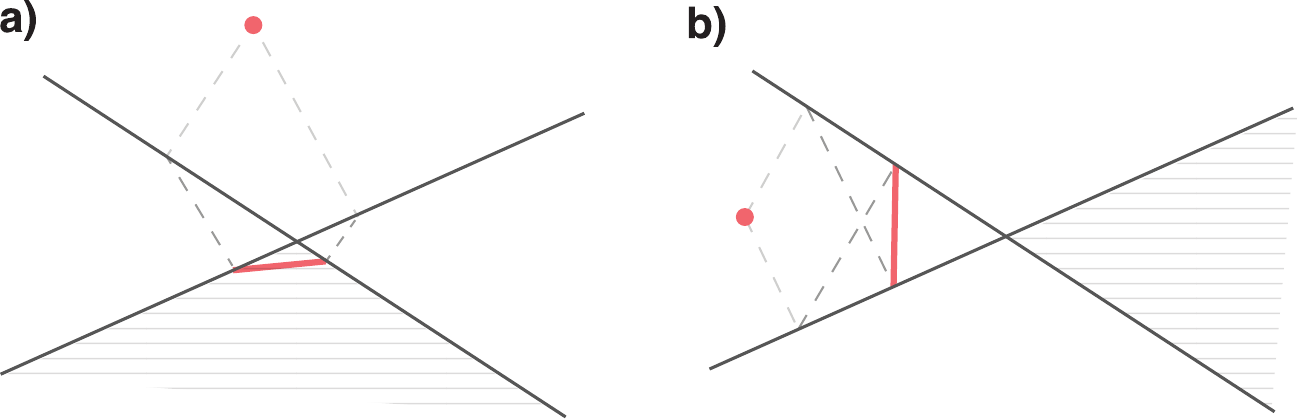}
    \caption{
        \textbf{Geometric intuition for compatibility and conflict between CBF constraints.}
        (\textbf{a}) Two CBF halfspaces with non-opposing normals
        ($A_i^\top A_j \ge 0$).
        Sequential projection onto one constraint does not violate the other, and the
        resulting control lies in the intersection of the halfspaces; convex mixing of
        such controls therefore preserves feasibility.
        (\textbf{b}) Conflicting CBF constraints with opposing normals
        ($A_i^\top A_j < 0$), where enforcing one constraint necessarily degrades the
        other.
        Such constraints cannot be said to form an antichain and must instead be
        ordered or prioritized, motivating poset-structured safety enforcement.
    }
    \label{fig:appendix_acuteness}
\end{figure}

\section{Experimental Details}
\label{app:experiment}

All quantitative results are reported over $N=100$ closed-loop evaluation rollouts.
For each rollout $k$, let $b(\bm{x}(t)) \ge 0$ denote the task-specific safety barrier function.

\textbf{Safety (min / mean).}
We report the minimum and mean safety margin over time and across rollouts,
defined as
\[
\text{Safety}_{\min} = \min_k \left\{ \min_{t \in [t_0, T]} b(\bm{x}(t)) \right\}, \quad
\text{Safety}_{\text{mean}}
= \mathbb{E}_k \left[ \mathbb{E}_{t\in[t_0,T]} \left[ b(\bm{x}^{(k)}(t)) \right] \right].
\]
Negative values indicate violation of the corresponding safety constraint.

\textbf{Top Safety Guarantee.}
This metric indicates whether the safety constraint is satisfied at all time steps
throughout the rollout.
For obstacle avoidance tasks, this corresponds to $\min_t b(x(t)) > 0$.
For joint-space safety, it indicates zero $\phi$-violation over the rollout.

\textbf{Uncertainty.}
Control uncertainty is quantified as the time-averaged standard deviation of control inputs
across rollouts, 
\[
\text{Unc.}(\bm{u}_i) = \mathbb{E}_{t} \left[ \mathrm{std}_k \left( \bm{u}_i^{(k)}(t) \right) \right], \quad i \in \{1,2\}.
\]

\textbf{Feasibility.}
The \textit{Feasibility} metric indicates whether a valid control input is
successfully produced at \emph{every} time step along a closed-loop rollout.
A rollout is deemed infeasible if the underlying safety layer fails to return
a control input at any time step, for example due to optimization infeasibility,
numerical instability, or solver failure.
This metric therefore captures the operational reliability of the safety
enforcement mechanism, independent of the achieved safety margin.

\subsection{2D Robot Multi-Obstacle Avoidance}

To evaluate robustness, we inject zero-mean uniform noise with magnitude $10\%$
of the control input scale during closed-loop execution.
For each of the 24 held-out test episodes, we perform multiple stochastic rollouts,
resulting in a total of $N=100$ evaluation runs.

\paragraph{Models.}
All models use fully connected networks with architecture
$[5, 128, 32, 32, 2]$ and \texttt{ReLU} activations. There are some additional layers of differentiable QPs in other models (other than E2E-related models). The model input consists of the system state and the goal.

\paragraph{Training and Dataset.}
For the 2D robot multi-obstacle avoidance task, the training set consists of
184 trajectories (66,420 state--control pairs), while the held-out test set
contains 24 trajectories (8,790 samples), each with approximately 360 time steps.
Ground-truth controls are generated by solving higher-order control barrier
function (HOCBF)-based quadratic programs that enforce all three obstacle-avoidance
constraints simultaneously.
All models are trained using the \texttt{Adam} optimizer with an MSE loss
and a learning rate of $10^{-3}$.
Differentiable QP layers are implemented using \texttt{qpth} from OptNet.
The training time for both \PoSafeNet{} variants (mixture and hard) is
approximately one minute for 20 epochs.
For multi-head models, ABNet uses 5 heads, while \PoSafeNet{} uses 6 heads.

\paragraph{Robot dynamics and safety constraints.}
We employ the unicycle model as the robot dynamics:
\begin{equation}
\underbrace{
\begin{bmatrix}
\dot{x}(t) \\
\dot{y}(t) \\
\dot{\theta}(t) \\
\dot{v}(t)
\end{bmatrix}
}_{\dot{\bm{x}}(t)}
=
\underbrace{
\begin{bmatrix}
v(t)\cos\theta(t) \\
v(t)\sin\theta(t) \\
0 \\
0
\end{bmatrix}
}_{f(\bm{x})}
+
\underbrace{
\begin{bmatrix}
0 & 0 \\
0 & 0 \\
1 & 0 \\
0 & 1
\end{bmatrix}
}_{g(\bm{x})}
\underbrace{
\begin{bmatrix}
u_1(t) \\
u_2(t)
\end{bmatrix}
}_{\bm{u}},
\label{eq:unicycle_dynamics}
\end{equation}
where $(x,y)\in\mathbb{R}^2$ denotes the planar position of the robot,
$\theta\in\mathbb{R}$ is the heading angle,
$v\in\mathbb{R}$ is the linear speed,
and $u_1$ and $u_2$ are the angular velocity and acceleration controls,
respectively.

The obstacle-avoidance safety constraint is defined as
\begin{equation}
b(\bm{x}) = (x - x_0)^2 + (y - y_0)^2 - R^2 \ge 0,
\label{eq:unicycle_safety}
\end{equation}
where $(x_0,y_0)\in\mathbb{R}^2$ denotes the obstacle center and $R>0$ is its radius.

\subsection{Safe Robot Manipulation under Joint Constraints}

To evaluate robustness, we inject zero-mean uniform noise with magnitude $10\%$
of the control input scale during closed-loop execution.
For each of the 20 held-out test episodes, we perform multiple stochastic rollouts,
resulting in a total of $N=100$ evaluation runs.

\paragraph{Models.}
All models use fully connected neural networks with architecture
$[6, 128, 256, 128, 32, 32, 2]$ and \texttt{ReLU} activation functions.
Unstructured safety baselines additionally include differentiable QP layers,
whereas the end-to-end (E2E) baseline does not enforce explicit safety constraints.

\paragraph{Training and Dataset.}
For the safe robot manipulation task, the training set consists of
160 trajectories with approximately 380 time steps each
(60,614 state--control pairs).
The held-out test set contains 20 trajectories with approximately
370 time steps per trajectory (7,329 samples).
Ground-truth controls are generated by solving higher-order control barrier
function (HOCBF)-based quadratic programs that enforce only end-effector
(obstacle-avoidance) safety, without joint-angle constraints.
All models are trained using the \texttt{Adam} optimizer with an MSE loss
and a learning rate of $10^{-3}$.
Differentiable QP layers are implemented using \texttt{qpth} from OptNet.
The training time for both \PoSafeNet{} variants (mixture and hard) is
approximately 24 seconds for 20 epochs.
For multi-head models, ABNet uses 5 heads, while \PoSafeNet{} uses 2 heads.

\paragraph{Robot dynamics and safety constraints.}
We model the two-link planar manipulator using the following control-affine dynamics:
\begin{equation}
\underbrace{
\begin{bmatrix}
\dot{\theta}_1 \\
\dot{\omega}_1 \\
\dot{\theta}_2 \\
\dot{\omega}_2
\end{bmatrix}
}_{\dot{\bm{x}}}
=
\underbrace{
\begin{bmatrix}
\omega_1 \\
0 \\
\omega_2 \\
0
\end{bmatrix}
}_{f(\bm{x})}
+
\underbrace{
\begin{bmatrix}
0 & 0 \\
1 & 0 \\
0 & 0 \\
0 & 1
\end{bmatrix}
}_{g(\bm{x})}
\underbrace{
\begin{bmatrix}
u_1 \\
u_2
\end{bmatrix}
}_{\bm{u}},
\label{eq:manipulator_dynamics}
\end{equation}
where $(\theta_1,\theta_2)\in\mathbb{R}^2$ denote the joint angles,
$(\omega_1,\omega_2)\in\mathbb{R}^2$ are the joint angular velocities,
and $u_1,u_2$ are the joint angular acceleration controls.

The end-effector obstacle-avoidance safety constraint is defined as
\begin{equation}
b_{\mathrm{tip}}(\bm{x})
=
(l_1 \cos\theta_1 + l_2 \cos\theta_2 - x_0)^2
+
(l_1 \sin\theta_1 + l_2 \sin\theta_2 - y_0)^2
- R^2
\ge 0,
\label{eq:tip_safety}
\end{equation}
where $(x_0,y_0)\in\mathbb{R}^2$ denotes the obstacle center,
$R>0$ is its radius, and $l_1,l_2>0$ are the link lengths.
In the current setting, ensuring end-effector safety also guarantees
non-collision of the links.

In addition to end-effector safety, we impose joint-angle constraints
to prevent excessive bending of the manipulator.
We define the relative joint angle
\begin{equation}
\phi = \mathrm{wrap}(\theta_2 - \theta_1) \in [-\pi,\pi),
\end{equation}
where $\mathrm{wrap}(\cdot)$ maps angles to the principal interval $[-\pi,\pi)$.
The corresponding angular velocity is $\dot{\phi} = \omega_2 - \omega_1$.

We enforce joint-angle limits by constraining $\phi$ to lie within a prescribed
interval $[\phi_{\min}, \phi_{\max}]$ using two higher-order control barrier
functions:
\begin{align}
b_{\phi}^{\min}(\bm{x}) &= \phi - \phi_{\min} \ge 0, \\
b_{\phi}^{\max}(\bm{x}) &= \phi_{\max} - \phi \ge 0 .
\end{align}
Since $\phi$ has relative degree two with respect to the control input,
with $\ddot{\phi} = u_2 - u_1$, each constraint is enforced via a
second-order control barrier function (HOCBF).

These joint-angle constraints are treated as higher-priority safety requirements
than the end-effector obstacle-avoidance constraint in the poset,
reflecting the fact that joint-limit violations are unacceptable even
when obstacle avoidance could otherwise be satisfied.

\subsection{Vision-Based End-to-End Autonomous Driving
}
\paragraph{Models.}
All methods share the same perception backbone, consisting of a convolutional
neural network followed by an LSTM and fully connected control heads.
The CNN comprises five convolutional layers with configurations
$[3,24,5,2,2]$, $[24,36,5,2,2]$, $[36,48,3,2,1]$, $[48,64,3,1,1]$, and $[64,64,3,1,1]$,
where each tuple specifies the number of input channels, output channels,
kernel size, stride, and padding, respectively.
The convolutional features are processed by an LSTM with hidden size 64 to
capture temporal dependencies.
The control head consists of two fully connected layers of size $[32,32,2]$
with \texttt{ReLU} activations, producing the steering rate and acceleration
commands.

Dropout with rate 0.3 is applied to both convolutional and fully connected
layers during training.
Unstructured safety baselines additionally include differentiable QP layers,
whereas the end-to-end (E2E) baseline does not incorporate explicit safety
constraints.
The model input is the front-view RGB image from the ego vehicle with resolution
$3\times45\times155$, and the output is a two-dimensional control vector
corresponding to steering rate and longitudinal acceleration.

\paragraph{Training and Dataset.}
We use a publicly available dataset containing approximately 0.4 million
image--control pairs collected in a closed-road sim-to-real driving environment.
The dataset is generated using the \texttt{VISTA} simulator~\citep{9812276},
and includes static and parked vehicles of varying types and appearances as
obstacles.
Ground-truth control commands are obtained by solving a nonlinear model
predictive control (NMPC) problem, which serves as the expert policy for
imitation learning.
All models are trained using the \texttt{Adam} optimizer with an MSE loss and a
learning rate of $10^{-3}$.
For methods incorporating differentiable safety layers, quadratic programs are
solved using \texttt{qpth} from OptNet~\citep{amos2017optnet}.
Training the \PoSafeNet{} requires approximately 50 minutes for 10 epochs on a
single NVIDIA RTX~5090 GPU.

\paragraph{Brief introduction to \texttt{VISTA}.}
\texttt{VISTA} is a sim-to-real driving simulator that generates diverse driving
scenarios from real-world driving data~\citep{9812276}.
We train vision-based policies via guided imitation learning using an expert
controller.
Data are generated in three steps:
(i) \texttt{VISTA} randomly initializes the poses of the ego vehicle and other
traffic participants (ado vehicles) consistent with the underlying real driving
logs;
(ii) a nonlinear model predictive controller (NMPC) is used to compute expert
controls and the corresponding state trajectories; and
(iii) front-view RGB images are recorded along NMPC rollouts and paired with
the expert control commands.
\paragraph{Vehicle dynamics and collision safety constraints.}
Vehicle dynamics are formulated in a curvilinear coordinate frame defined with
respect to a reference trajectory (e.g., the lane centerline)~\citep{rucco2015efficient}.
The system state is
$x = [s, d, \mu, v, \delta]^\top$, where $s\in\mathbb{R}$ denotes the progress
along the reference trajectory, $d\in\mathbb{R}$ is the lateral offset of the
vehicle center, $\mu$ is the heading error with respect to the reference
trajectory, $v$ is the longitudinal speed, and $\delta$ is the steering angle.
The control input is $u = [u_1, u_2]^\top$, where $u_1$ is the steering rate and
$u_2$ is the longitudinal acceleration.

We adopt the curvilinear bicycle model in~\citep{rucco2015efficient}, which can be written
in control-affine form as
\begin{equation}
\label{eq:vista_dynamics_matrix}
\underbrace{
\begin{bmatrix}
\dot{s} \\
\dot{d} \\
\dot{\mu} \\
\dot{v} \\
\dot{\delta}
\end{bmatrix}
}_{\dot{\bm{x}}}
=
\underbrace{
\begin{bmatrix}
\dfrac{v\cos(\mu+\beta)}{1-d\kappa(s)} \\[6pt]
v\sin(\mu+\beta) \\[6pt]
\dfrac{v}{\ell_r}\sin\beta
-\dfrac{\kappa(s)\,v\cos(\mu+\beta)}{1-d\kappa(s)} \\[6pt]
0 \\[2pt]
0
\end{bmatrix}
}_{f(\bm{x})}
+
\underbrace{
\begin{bmatrix}
0 & 0 \\
0 & 0 \\
0 & 0 \\
1 & 0 \\
0 & 1
\end{bmatrix}
}_{g(\bm{x})}
\underbrace{
\begin{bmatrix}
u_1 \\
u_2
\end{bmatrix}
}_{\bm{u}}.
\end{equation}

Collision avoidance is encoded via a control barrier function defined in the
curvilinear frame:
\begin{equation}
b_{\mathrm{col}}(\bm{x})
= (s-s_0)^2 + (d-d_0)^2 - R^2 \ge 0,
\label{eq:vista_collision_barrier}
\end{equation}
where $(s_0,d_0)\in\mathbb{R}^2$ denotes the obstacle location expressed in
curvilinear coordinates and $R>0$ is a safety radius chosen such that satisfying
\eqref{eq:vista_collision_barrier} guarantees collision avoidance.

\paragraph{Lane-keeping constraints.}
Lane keeping is enforced by constraining the lateral offset $d$ of the ego
vehicle with respect to the reference trajectory.
Specifically, we require the vehicle to remain within a lateral margin
$\ell_f>0$ around the lane centerline:
\begin{equation}
-\ell_f \le d \le \ell_f .
\end{equation}

This is implemented using two higher-order control barrier functions
corresponding to the left and right lane boundaries.
For the left boundary, we define
\begin{equation}
b_{\mathrm{L}}(\bm{x}) = \ell_f - d \ge 0,
\end{equation}
and for the right boundary,
\begin{equation}
b_{\mathrm{R}}(\bm{x}) = d + \ell_f \ge 0.
\end{equation}
Both constraints are enforced using second-order barrier conditions
that account for the vehicle dynamics in the curvilinear coordinate frame,
resulting in affine constraints on the control input at each time step.

\paragraph{Closed-loop testing.}
All methods are evaluated closed-loop in \texttt{VISTA}.
At each time step, the policy receives a front-view RGB observation and outputs
a control command, which is applied to the simulator.
Rollouts terminate when the horizon is reached (or upon collision, when applicable).
Unless otherwise stated, we report results over $N=100$ closed-loop runs.

For each run, the ego vehicle is initialized on the straight segment with
$s_0 \sim \mathcal{U}(0,10)\,\mathrm{m}$ and lateral offset
$d \sim \mathcal{U}(-0.5,0.5)\,\mathrm{m}$.
Two obstacles are placed ahead of the ego vehicle with longitudinal offsets
$\Delta s_1 \sim \mathcal{U}(20,30)\,\mathrm{m}$ and
$\Delta s_2 = \Delta s_1 + \mathcal{U}(30,40)\,\mathrm{m}$,
and lateral offsets sampled uniformly from
$\{|d_0| \sim \mathcal{U}(0.1,1.5)\}\,\mathrm{m}$ with random sign.

\section{Additional Ablations on Order Sensitivity and Control Quality (VISTA)}
\label{app:ablation}
\paragraph{Effect of Constraint Ordering.}
To study the sensitivity of sequential safety projection to constraint ordering,
we compare four variants that differ only in how lane-keeping and obstacle-avoidance
constraints are ordered within the projection pipeline.

In \textbf{\PoSafeNet{} (Fix Order)}, obstacle avoidance is enforced as the highest-priority
constraint, followed by the left and right lane boundary constraints in a fixed,
manually specified order. This design reflects an intuitive safety hierarchy in which
collision avoidance strictly dominates lane keeping.

In \textbf{\PoSafeNet{} (Wrong Order)}, this hierarchy is intentionally violated by enforcing
lane-keeping constraints above obstacle avoidance. This ablation isolates the effect
of incorrect cross-priority ordering, where control authority is prematurely consumed
by lower-priority constraints, leaving insufficient flexibility to satisfy
collision-avoidance objectives.

Although both variants employ the same set of safety constraints, swapping the relative
priority between lane keeping and obstacle avoidance leads to qualitatively different
outcomes (Table~\ref{tab:ablation_benchmark_vista}). In particular, the wrong-order
variant exhibits severe degradation in safety margins and collision performance,
demonstrating that incorrect ordering across priority levels can fundamentally
compromise safety, even when all constraints are individually feasible.

Interestingly, \textbf{\PoSafeNet{} (Fix Order)} remains collision-free with positive
minimum barrier values, but incurs substantial lane violations, as evidenced by highly
negative lane margins and increased time spent outside lane boundaries. This behavior
reflects a structural limitation of rigid, hand-specified ordering: by always enforcing
obstacle avoidance first, control authority is exhausted early in the projection
pipeline, forcing lower-priority lane constraints to be satisfied only through large
lateral deviations.

In contrast, \textbf{\PoSafeNet{} (Hard)} learns a global projection strategy via
Gumbel-Hard selection over multiple candidate heads, each corresponding to a distinct
ordering consistent with the poset structure. By adaptively selecting the ordering that
best matches the data distribution, this learned strategy improves safety margins and
lane adherence relative to fixed-order baselines.

Finally, \textbf{\PoSafeNet{} (Mixture)} softly combines multiple heads to further enhance
robustness, yielding smoother control behavior, reduced lateral deviation, lower
rotation rates, and less time spent outside lane boundaries, while consistently
maintaining collision-free operation.

\begin{table*}[t]
\centering
\caption{
\textbf{Ablation benchmark results on VISTA vision-based autonomous driving via \PoSafeNet{}.}
Results are reported over the held-out test set under identical evaluation protocols
(see Appendix~\ref{app:experiment}).
Safety is quantified by the minimum and mean barrier value over each rollout, where negative values indicate safety violation.
Lane keeping performance is measured by the minimum and mean lane margin (lower is better),
with negative values indicating lane departure.
We additionally report the mean and variance of the absolute lateral deviation $|e_y|$ from the lane centerline to quantify tracking accuracy and lateral stability,
as well as the fraction of time the vehicle exceeds lane boundaries (\textit{time-out-of-lane}) as a hard lane violation metric.
\textit{Pass\%} indicates the percentage of rollouts completed without collision.
\textit{Rot viol.} measures the mean rotation rate magnitude, and
\textit{Uncertainty (u1,u2)} denotes the standard deviation of control inputs over time.
Higher is better for Pass and Safety; lower is better for lane violation, $|e_y|$, time-out-of-lane, rotation violation, and uncertainty.
}
\label{tab:ablation_benchmark_vista}
\resizebox{\textwidth}{!}{%
\begin{tabular}{lcccccccc}
\toprule
Method
& Pass\% $\uparrow$
& Crash\% $\downarrow$
& Safety (min / mean) $>0$
& Lane viol.\ (min / mean) $\downarrow$
& $|e_y|$ (mean / var) $\downarrow$
& Time-out-of-lane $\downarrow$
& Rot viol.\ $\downarrow$
& Unc.\ (u1,u2) \\
\midrule

PoSafeNet (Fix Order)
& 100.0
& 0.0
& $0.27$ / $3.38$
& $-6.15$ / $-1.71$
& $1.73$ / $3.09$
& 0.3197
& 0.65
& $(0.412,\;0.641)$ \\

PoSafeNet (Wrong Order)
& 41.0
& 59.0
& $-24.42$ / $-17.70$
& $0.06$ / $0.14$
& $0.989$ / $0.746$
& 0.0000
& 0.00
& $(0.108,\;0.742)$ \\

PoSafeNet (Mixture)
& 100.0
& 0.0
& $1.23$ / $8.21$
& $-0.93$ / $-0.49$
&  $1.028$ / $1.278$
&  0.1992
&  0.14
& $(0.287,\;0.387)$ \\

PoSafeNet (Hard)
& 100.0
& 0.0
& $2.73$ / $7.12$
& $-0.63$ / $-0.32$
& $1.404$ / $1.599$
& 0.3670
& 0.80
& $(0.345,\;0.456)$ \\
\bottomrule
\end{tabular}}
\end{table*}

\section{Four-Level Poset Safety Hierarchy for Heatmap Visualization}
\label{app:four_level_poset}

This appendix describes the four-level safety hierarchy used to generate the
heatmap visualizations in the vision-based driving experiments.
The purpose is not to introduce new modeling assumptions or additional safety
constraints, but to illustrate how refining the priority structure within the
same poset-based framework influences closed-loop behavior.
In particular, this example highlights the ability of \PoSafeNet{} to incorporate
additional low-priority preferences without compromising higher-priority safety
requirements.

\subsection{Poset Structure}

Figure~\ref{fig:task_hasse_2} compares the task-specific safety posets used in the
experiments.
The two-level hierarchy enforces only road boundary constraints and obstacle
avoidance, whereas the four-level hierarchy further decomposes driving behavior
into multiple priority tiers that encode softer preferences.

\begin{figure}[t]
    \centering
    \includegraphics[width=0.3\linewidth]{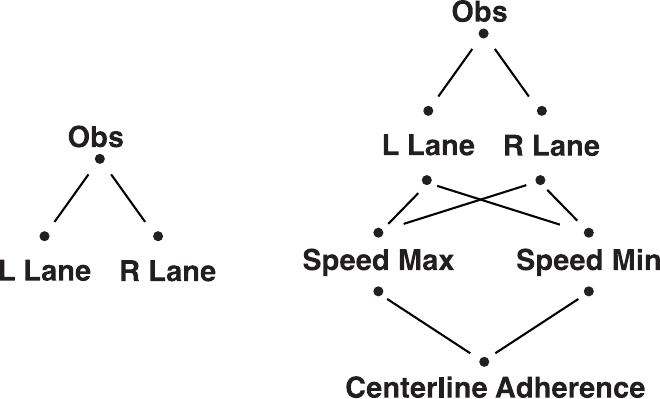}
    \caption{\textbf{Task-specific safety posets used in the vision-based driving experiments,
    visualized as Hasse diagrams.}
    Left: a two-level hierarchy where obstacle avoidance has higher priority than
    left and right lane boundary constraints.
    Right: a four-level hierarchy that augments the two-level structure with
    intermediate speed regulation and low-priority centerline adherence, while
    preserving obstacle avoidance as the highest-priority constraint.}
    \label{fig:task_hasse_2}
\end{figure}

Formally, we use the partial order symbol $\preceq$ to denote \emph{priority
relations} between constraints, where
$c_1 \preceq c_2$ indicates that constraint $c_1$ has lower or equal priority than
constraint $c_2$.
Under this convention, the four-level hierarchy orders constraints from lowest
to highest priority as
\begin{equation}
\text{Centerline} \;\preceq\; \text{Speed} \;\preceq\; \text{Lane} \;\preceq\; \text{Obstacle}.
\end{equation}

Low-priority constraints, such as centerline adherence and speed regulation, are
treated as soft behavioral preferences.
They are enforced only when doing so does not interfere with higher-priority
safety constraints, such as lane boundaries and obstacle avoidance.

\subsection{Additional Barrier Definitions}

In addition to the lane boundary and obstacle avoidance barriers defined
elsewhere, the four-level hierarchy introduces two \emph{additional low-priority}
constraints to better capture common driving preferences.

\paragraph{Centerline adherence constraints.}
Centerline adherence is modeled as a soft preference that encourages the vehicle
to remain close to the lane centerline.
Let $d$ denote the signed lateral deviation from the lane centerline and
$\tau > 0$ a prescribed tolerance.
We define two symmetric barrier functions
\begin{equation}
b_{\mathrm{CL},L}(\bm{x}) = \tau - d \ge 0, \qquad
b_{\mathrm{CL},R}(\bm{x}) = d + \tau \ge 0.
\end{equation}
These barriers share the same higher-order control barrier function (HOCBF)
structure as the lane boundary constraints, but are enforced at a strictly lower
priority level.

\paragraph{Speed regulation constraints.}
Speed regulation is introduced as a low-priority constraint to discourage
overspeeding and excessive braking.
Let $v$ denote the vehicle speed and $v_{\min}, v_{\max}$ the desired bounds.
We define first-order control barrier functions
\begin{equation}
b_{\max}(v) = v_{\max} - v \ge 0, \qquad
b_{\min}(v) = v - v_{\min} \ge 0.
\end{equation}
These constraints yield affine conditions on the longitudinal acceleration and
are enforced at a lower priority than lane boundary and obstacle avoidance
constraints.

\subsection{Interpretation of Heatmaps}

\begin{figure}[t]
    \centering
    \includegraphics[width=0.75\linewidth]{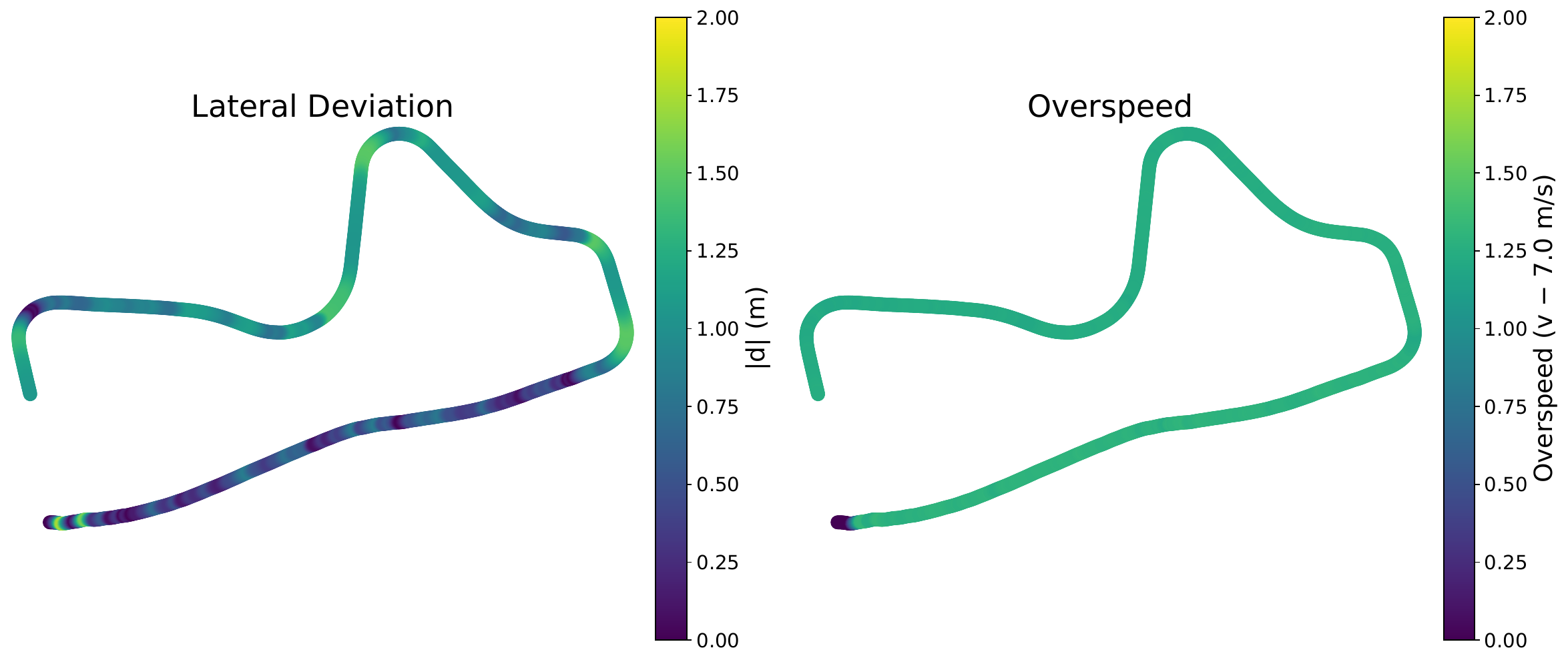}

    \vspace{0.5em}

    \includegraphics[width=0.75\linewidth]{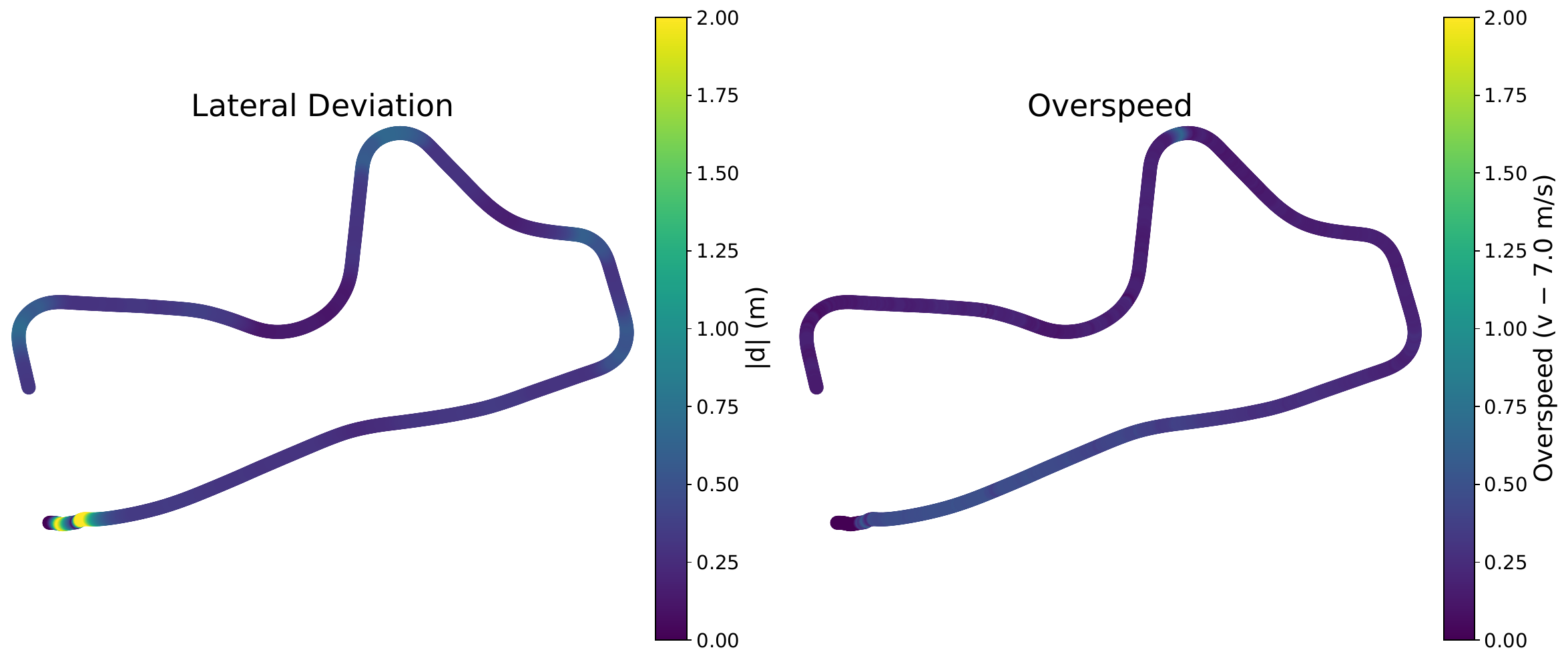}
    \caption{\textbf{Heatmap visualization of closed-loop driving behavior under different
    safety hierarchies.}
    Top: two-level poset (lane $\preceq$ obstacle).
    Bottom: four-level poset (centerline $\preceq$ speed $\preceq$ lane $\preceq$ obstacle).
    Colors indicate lateral deviation and overspeed along the trajectory.}
    \label{fig:heatmap_poset_vertical}
\end{figure}

The heatmaps illustrate how introducing additional low-priority structure alters
the distribution of closed-loop behaviors.
Compared to the two-level hierarchy, the four-level hierarchy exhibits reduced
lateral deviation and lower overall speed, while preserving collision avoidance
and lane safety.

These results demonstrate that \PoSafeNet{} can absorb additional behavioral
preferences by refining the poset hierarchy, shaping closed-loop behavior in a
more structured and human-like manner without compromising higher-priority
safety constraints.
This example therefore serves as a qualitative illustration of the extensibility
of the poset-based formulation.

\section{Supplementary On Vision Based End-to-End Autonomous Driving}
This section provides qualitative visualizations of closed-loop driving behaviors in the VISTA simulator for representative methods. Each snapshot shows the front-view RGB observation received by the policy (left) and a corresponding top-down view of the scene (right), including the ego vehicle, lane centerline, and surrounding obstacles.

These visualizations are intended to complement the quantitative results reported in the main paper by illustrating typical driving behaviors under
dynamic obstacle interactions, rather than to serve as a standalone performance comparison.
\begin{figure}[t]
    \centering
    \includegraphics[width=0.65\linewidth]{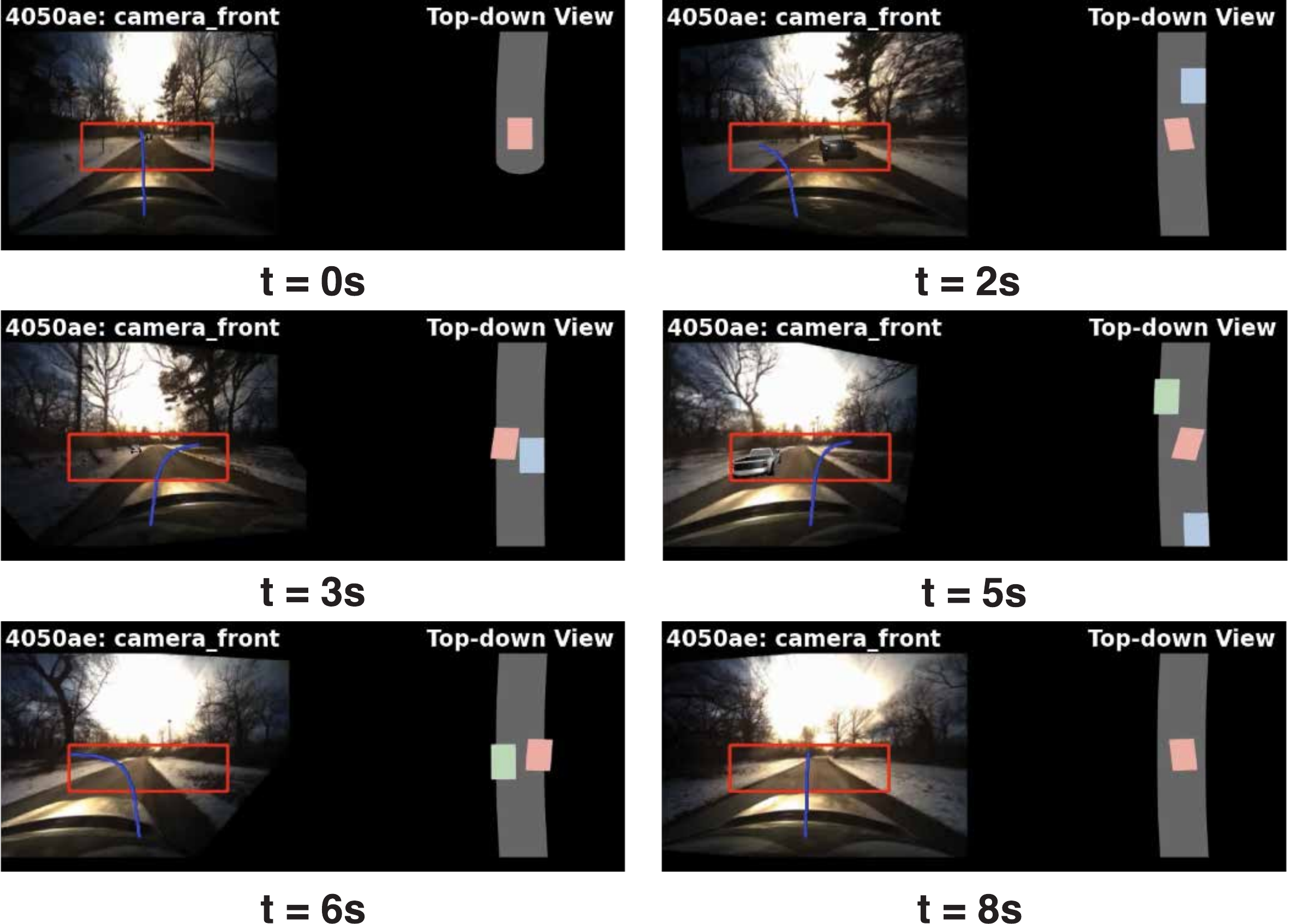}

    \par\bigskip

    \includegraphics[width=0.65\linewidth]{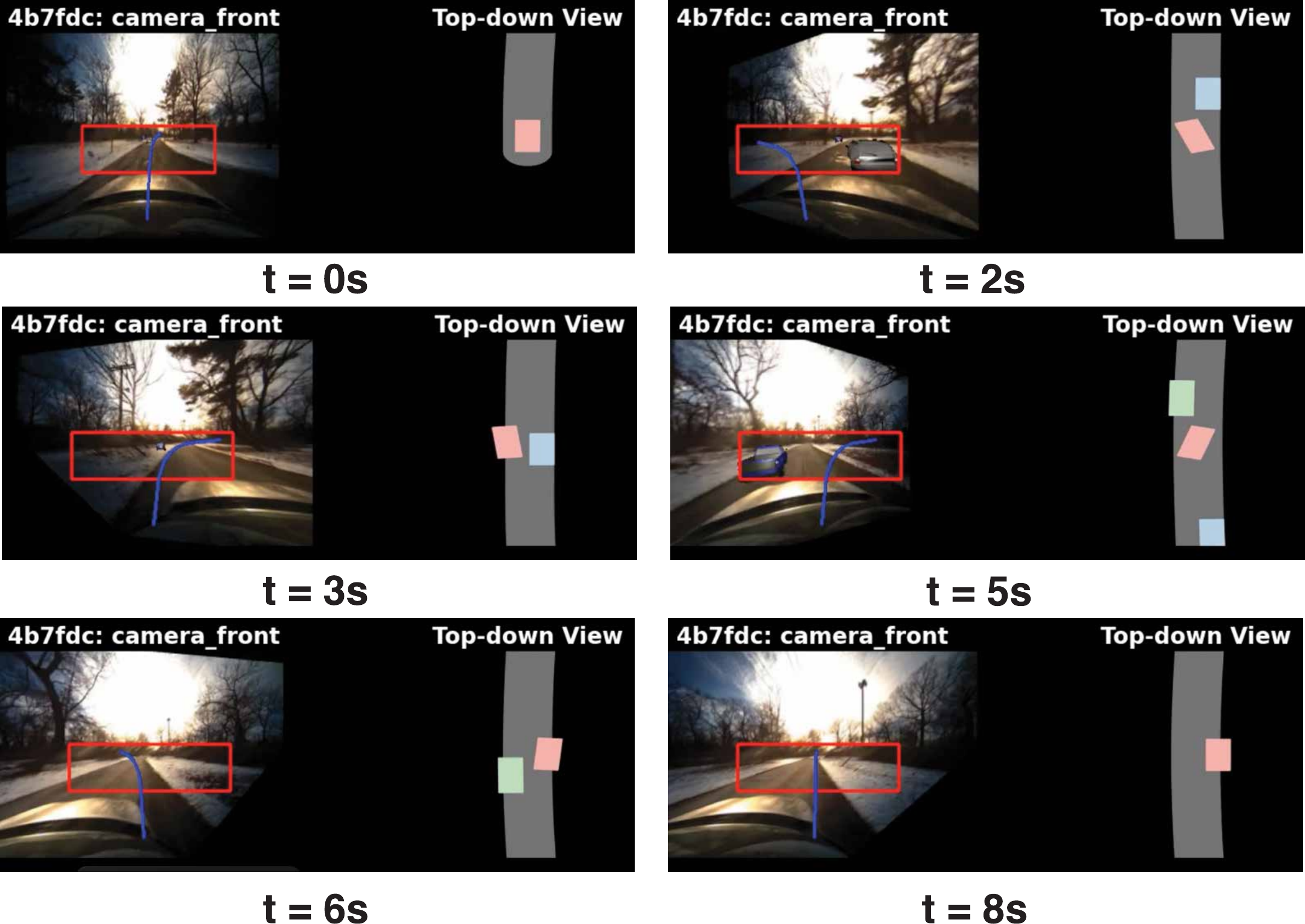}

    \caption{\textbf{Qualitative visualization of closed-loop driving behaviors in the
    VISTA simulator.}
    Each snapshot corresponds to a specific time step along the trajectory
    (timestamps shown above each panel).
    For each time step, the left image shows the front-view camera observation,
    while the right image shows the corresponding top-down view.
    From top to bottom: \PoSafeNet{} (Mixture) and \PoSafeNet{} (Hard).}
    \label{fig:vista_posafenet}
\end{figure}

\begin{figure}[t]
    \centering
    \includegraphics[width=0.65\linewidth]{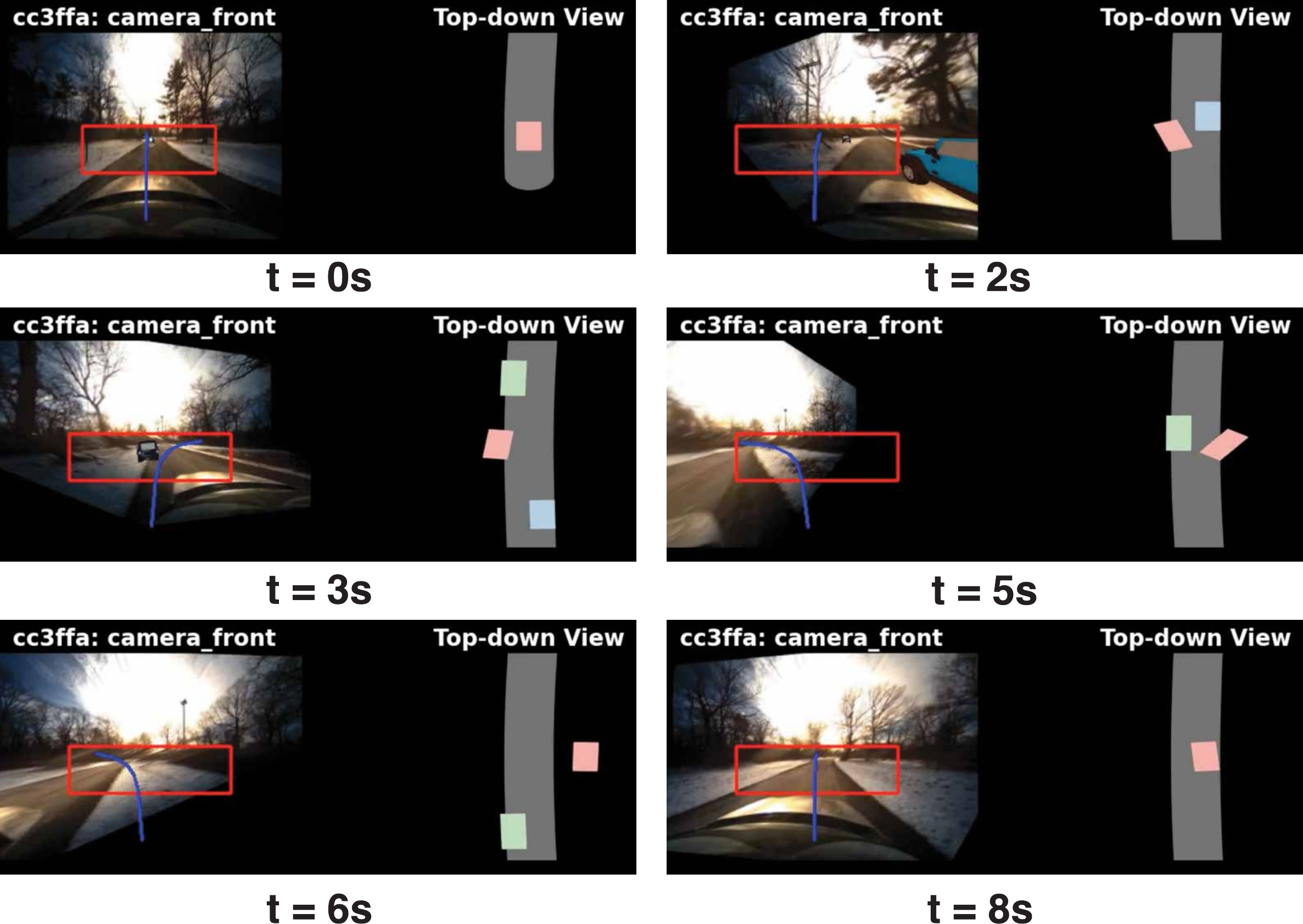}
    \caption{\textbf{Qualitative visualization of closed-loop driving behaviors in the
    VISTA simulator using ABNet.}
    Each snapshot corresponds to the same set of time steps as in
    Fig.~\ref{fig:vista_posafenet} (timestamps shown above each panel).
    The left image shows the front-view camera observation and the right image
    shows the corresponding top-down view.
    Compared with \PoSafeNet{}, ABNet without lane keeping constraint exhibits less consistent lane tracking.}
    \label{fig:vista_abnet}
\end{figure}

\end{document}